%% file: afm-main.tex
\documentclass[preprint]{imsart}
\input{_preamble.tex}

\begin{document}

\begin{frontmatter}

\title{A Bayesian sparse factor model with adaptive posterior concentration\support{This work was supported by the National Research Foundation of Korea (NRF) grant funded by the Korea government (MSIT)
 (No. 2020R1A2C3A01003550 and No.  2022R1F1A1069695). Lizhen Lin would like to acknowledge the generous support of NSF grants DMS CAREER 1654579 and DMS 2113642.}}
\runtitle{Adaptive Bayesian factor model}

\begin{aug}
\author[A]{\fnms{Ilsang} \snm{Ohn}\ead[label=e1]{ilsang.ohn@inha.ac.kr}},
\author[B]{\fnms{Lizhen} \snm{Lin}\ead[label=e2]{lizhen.lin@nd.edu}}
\and
\author[C]{\fnms{Yongdai} \snm{Kim}\ead[label=e3]{ydkim0903@gmail.com}}

\address[A]{Department of Statistics, Inha University \printead{e1}}
\address[B]{Department of Applied and Computational Mathematics and Statistics, The University of Notre Dame \printead{e2}}
\address[C]{Department of Statistics, Seoul National University \printead{e3}}
\end{aug}

\runauthor{I. Ohn, L. Lin and Y. Kim}

\begin{abstract}
In this paper, we propose a new Bayesian inference method for a high-dimensional sparse factor model that allows both the factor dimensionality and the sparse structure of the loading matrix to be inferred. The novelty is to introduce a certain dependence between the sparsity level and the factor dimensionality, which leads to adaptive posterior concentration while keeping computational tractability. We show that the posterior distribution asymptotically concentrates on the true factor dimensionality, and more importantly, this posterior consistency is adaptive to the sparsity level of the true loading matrix and the noise variance. We also prove that the proposed Bayesian model attains the optimal detection rate of the factor dimensionality in a more general situation than those found in the literature. Moreover, we obtain a near-optimal posterior concentration rate of the covariance matrix. Numerical studies are conducted and show the superiority of the proposed method compared with other competitors. 
\end{abstract}


\end{frontmatter}

\section{Introduction}

In this paper, we propose a novel Bayesian method for learning a high-dimensional sparse linear factor model and study asymptotic concentration properties of the posterior distribution. We consider the following linear factor model where  $p$-dimensional random vectors $\Y_1,\dots,\Y_n$ are distributed as
    \begin{equation}
    \label{eq:model}
    \Y_i|(\Z_i=\z_i) \indsim \N(\B\z_i, \psi\I), \quad \Z_i\iidsim \N(\zero, \I)
    \end{equation}
for $i\in[n]$ with $\B$ representing a $p\times \xi$ factor loading matrix, $\psi>0$ a noise variance and $\Z_i$ a $ \xi$-dimensional (latent) factor related to datum $\Y_i$, where $\xi\in\{1,\dots, p-1\}$. Under this model, the marginal distribution of the data $\Y_1,\dots, \Y_n$ is given by
    \begin{equation*}
        \Y_i\iidsim \N(\zero, \bSigma), \quad \bSigma := \B\B^\top + \psi\I.
    \end{equation*}
Therefore, correlations among the observed variables in each $\Y_i$ in the above factor model are explained by a low rank matrix $\B\B^\top$ which leads to  a substantial but efficient reduction of the model complexity. The factor model has been used in a broad range of high-dimensional inference tasks including covariance estimation \citep{fan2008high, fan2011high, fan2018large}, linear regression \citep{bai2006confidence, kneip2011factor, stock2002forecasting}, multiple testing under arbitrary dependence \citep{fan2012estimating, fan2018farmtest, leek2008general} and other supervised learning tasks \citep{fan2017sufficient,silva2011two}.

For high-dimensional data where the dimension $p$ is much large than the sample size $n,$ we need a low dimensional structure on $\B$ to have a consistent estimator and \textit{sparsity} is a popularly used condition, which assumes that true loading matrix $\B^\star$ is sparse in the sense that only few entries of $\B^\star$ is nonzero and the other entries are zero. There are various Bayesian models for sparse linear factor models including 
\cite{bhattacharya2011sparse, srivastava2017expandable,xie2018bayesian,ning2021spike}.

Along with considering the sparsity, determining the \textit{factor dimensionality} $\xi$ is an important and practical topic in factor modeling.  From a theoretical point of view,  an appropriate estimation of the factor dimensionality is required to optimize the bias-variance trade-off in a factor model. The factor dimensionality is also of practical interest, especially when it has a physical interpretation e.g.,  the number of interacting pathways in genomics \citep{carvalho2008high} and the number of personality traits in psychology \citep{caprara1993big}.

Frequentist approaches typically adopt a two-step procedure where the factor dimensionality is chosen or estimated before estimating the parameters in the model. Many consistent model selection methods have been proposed, which fit the factor models for different values of $ \xi$ and select the best $ \xi$ based on their choice of model selection criteria \citep{bai2002determining, bai2007determining}. Alternatively, the eigenvalues of the empirical covariance or correlation matrix can be used to estimate the factor dimensionality. Several procedures related to this approach have been proposed and yielded consistency \citep{ahn2013eigenvalue,fan2020estimating, lam2012factor, onatski2010determining}. The estimation of the factor dimensionality for high-dimensional sparse factor models has also been considered by \cite{cai2013sparse,cai2015optimal}.

On the Bayesian side, prior distributions which put prior mass on $\xi$ directly have been popularly used to infer the factor dimensionality. Examples are the spike and slab prior with the Poisson prior on $\xi$ \citep{pati2014posterior} and spike and slap priors with the Indian buffet process (IBP) \citep{chen2010bayesian, knowles2011nonparametric, rovckova2016fast, ohn2020posterior}.  Recently, a few works have provided theoretical results for the posterior distribution of the factor dimensionality under nonparametric priors. \citet{rovckova2016fast} considered a spike and slab prior with the one-parameter IBP and proved that the posterior probability of the factor dimensionality being upper bounded by a certain quantity converges to 1. But the upper bound is much larger than the true factor dimensionality. \citet{ohn2020posterior} derived posterior consistency of the factor dimensionality  under a spike and slab prior with the two-parameter IBP. However, their result is nonadaptive in the sense that the choice of hyperparameters of the prior distribution relies on the information or knowledge of the true sparsity of the loading matrix, which is unknown in practice. 

Another promising theoretical result was provided by \citet{gao2015rate}. The authors studied the Bayesian factor model in the context of sparse principle component analysis (PCA) in which the prior distribution concentrates its mass on the orthogonal loading matrix, i.e., $\B^\top\B$ is diagonal, and established the adaptive posterior consistency of the factor dimensionality. However, the orthogonal constraint makes it hard to compute posterior distribution and \citet{gao2015rate} only succeeded in implementing a posterior sampler for the factor model with a one-dimensional factor i.e., $\xi=1$. 

We propose a novel Bayesian model that overcomes the theoretical and practical limitations of the existing Bayesian approaches. A key feature of the proposed Bayesian model is that the sparsity and the factor dimensionality are negatively correlated under the prior distribution, and this a priori negative correlation between them helps to prevent overestimating the true factor dimensionality. This is a critical difference between the proposed prior and the widely used IBP-type priors. Yet, posterior computation can be carried out through a simple and  efficient Monte Carlo Markov chain (MCMC) algorithm. Our numerical studies show that the developed MCMC algorithm can apply to high-dimensional data without many hampers.
 
We thoroughly investigate the theoretical properties of the posterior distribution of the proposed Bayesian model.  We prove that the posterior distribution of the factor dimensionality converges to the true one. In particular, we prove that the proposed Bayesian model attains the optimal detection rate for the eigengap (i.e., the size of the smallest eigenvalue of the low rank part $\B^\star(\B^\star)^\top$ of the true covariance matrix $\bSigma^\star=\B^\star(\B^\star)^\top+\psi^\star\I$) for the consistency of the factor dimensionality.
We also show that the posterior distribution of the covariance matrix concentrates around the truth at a near optimal rate. 
The novelty of our results lies in that it does not require any prior knowledge of the true sparsity and noise variance and hence all the optimal theoretical properties of the posterior are \textit{adaptive} to the sparsity and noise variance. 

It should be noted  that the proposed Bayesian model has theoretical advantages over not only other Bayesian factor models but also existing frequentist's estimators of the factor dimensionality. The estimator of \citet{cai2013sparse} is adaptive to the true sparsity but requires a larger detection rate for the eigengap than the optimal one. On the other hand, the estimator of \citet{cai2015optimal} achieves the optimal detection rate but is not adaptive to the true sparsity. Moreover, both estimators assume the known noise variance, which limits their applicability.

The rest of the paper is organized as follows. In \cref{sec:method}, we introduce the proposed prior distribution and develop an efficient MCMC algorithm for sampling from the posterior distribution. In \cref{sec:theory},  asymptotic properties of the posterior distribution are derived. In \cref{sec:numerical}, we conduct simulation studies and real data analysis. \cref{sec:conclusion} concludes the paper. 

\subsection{Notation}

Let $\R$, $\R_+$ and $\bN$ be the sets of real numbers, positive numbers and natural numbers, respectively.  Let $\zero$ and $\one$  denote vectors of 0's and of 1's, respectively, where the dimensions of such vectors can differ according to the context. For a positive integer $p$, we let $[p]:=\{1,2,\dots,p\}$. For a real number $x$, $\floor{x}$ denote the largest integer less than or equal to $x$ and $\ceil{x}$ denote the smallest integer larger than or equal to $x$.  For two real numbers $a$ and $b$, we write $a\vee b:=\max\{a,b\}$ and $a\wedge b:=\min\{a,b\}$. For two positive sequences $\{a_n\}_{n\in \mathbb{N}}$ and $\{b_n\}_{n\in \mathbb{N}}$ we write $a_n\lesssim b_n$ or equivalently $b_n\gtrsim a_n$ if there exists a positive constant $C>0$ such that $a_n\le Cb_n$ for any $n\in \mathbb{N}$. Moreover, we write $a_n\asymp b_n$ if both $a_n\lesssim b_n$  and  $a_n\gtrsim b_n$ hold. We denote by $\ind(\cdot)$ the indicator function.

For a set $\cS$, $|\cS|$ denotes its cardinality. For a $p$-dimensional vector $\bbeta:=(\beta_j)_{j\in[p]}$, let $\|\bbeta\|_r:=\del[0]{\sum_{j=1}^p|\beta_j|^r}^{1/r}$ for  $r\ge1$ and $\|\bbeta\|_0:=\sum_{j=1}^p\ind(\beta_j\neq0)$. For a set $\cS\subset\{1,\dots,p\}$, define $\bbeta_{[\cS]}:=(\beta_j)_{j\in \cS}$.  For a $p\times q$-dimensional matrix  $\A:=(a_{jk})_{j\in [p], k\in[q]}$, we denote the spectral norm of the matrix $\A$ by $\norm{\A}$ and the Frobenius norm by $\fnorm{\A}$, that is, $\norm{\A}:=\sup_{\x\in\R^q:\|\x\|_2=1}\|\A\x\|_2$ and $\fnorm{\A}:=\sqrt{\tr(\A^\top\A)}$. Let $\norm{\A}_1$ be the vector $\ell_1$ norm of $\A$, i.e., $\norm{\A}_1:=\sum_{j=1}^p\sum_{k=1}^q|a_{jk}|$. For sets $\cS\subset[p]$ and $\cK\subset[q]$, we let $\A_{[\cS, \cK]}:=(a_{jk})_{j\in\cS, k\in\cK}$ which is the submatrix of $\A$ taking the rows in $\cS$ and columns in $\cK$. For notational simplicity, we write  $\A_{[:, \cK]}:=\A_{[[p], \cK]}$ and $\A_{[\cS, :]}:=\A_{[\cS, [q]]}$. Furthermore,  let $\A_{[j, :]}:=\A_{[\{j\}, :]}$ and $\A_{[:, k]}:=\A_{[:, \{k\}]}$, which denote the $j$-th row and $k$-th column of $\A$, respectively. We let $\lambda_1(\bSigma)\ge\lambda_2(\bSigma)\dots\ge\lambda_p(\bSigma)$ be the ordered eigenvalues and $|\bSigma|$ be the determinant of a $p\times p$-dimensional matrix $\bSigma$.  Let $\bS_{++}^p$ be a set of $p\times p$ symmetric positive definite matrices. 

For a given probability measure $G$, let $\P_G$ denote the probability or the expectation operator under the probability measure $G$. We denote by $\sp_G$ the probability density function of $G$ with respect to the Lebesgue measure if exists. For convenience, we write $\P_{\bSigma}:=\P_{\N(\zero, \bSigma)}$ and $\sp_{\bSigma}:=\sp_{\N(\zero, \bSigma)}$ for a normal distribution $\N(\zero, \bSigma)$. For $n\in\bN$, let $\P_G^{(n)}$ be the probability or the expectation under the product measure and, if exists,  $\sp_G^{(n)}$ its density function.

\section{Proposed prior and MCMC algorithm}
\label{sec:method}

In this section, we  design a novel prior tailored for the loading matrix in a factor model and develop a computationally efficient MCMC algorithm for sampling the posterior distribution. 

\subsection{Prior distribution}


The proposed prior on the loading matrix first samples the ``sparse structure'' of the loading matrix and then samples nonzero elements. Let  $\u:=(u_1,\dots, u_p)^\top\in\Delta_p:=\{0,1\}^p\setminus\{\zero\}$ and, for  a positive integer $q\in\bN$ we have chosen, $\v:=(v_1,\dots, v_{q})^\top\in\Delta_{q}:=\{0,1\}^q\setminus\{\zero\}$. They are latent indicator variables that determine nonzero rows and columns of the loading matrix $\B$,  respectively. Note that the sparsity of columns determines the factor dimensionality $\xi$, i.e., $\xi =\|\v\|_0$ and $q$ is a pre-specified upper bound of the factor dimensionality. We impose the prior distribution on $\u$ and $\v$ such that
        \begin{equations}
        \label{eq:prior_ind}
            \Pi(\u, \v)=& 
            \frac{Q_{A}(\|\u\|_0,\|\v\|_0)}{\sum_{\u'\in\Delta_{p}}\sum_{\v'\in\Delta_{q}}Q_{A}(\|\u'\|_0,\|\v'\|_0)}\ind(\u\in\Delta_p, \v\in\Delta_{q}) \\
            &\mbox{ with } Q_{A}(\omega,\xi):=Q_{A,p,q,n}(\omega,\xi):=\frac{1 }{\binom{p}{\omega}\binom{q}{\xi}}\exp({-A\omega\xi\log (p\vee n)})
        \end{equations}
for some $A>0$. Under this prior distribution, the non-sparsity $\omega=\|\u\|_0$ and the factor dimensionality $\xi=\|\v\|_0$ are negatively correlated in the sense that $\omega$ becomes smaller when $\xi$ is large and vice versa.  Note that $\|\u\|_0$ and $\|\v\|_0$ are negatively correlated in the proposed prior (\ref{eq:prior_ind}), which sharply contrasts with existing  IBP-type priors \citep{rovckova2016fast, ohn2020posterior} that assume the independence of $\|\u\|_0$ and $\|\v\|_0.$ The IBP-type priors, however, are known to have the posterior consistency only when the true sparsity is known. We devise the prior (\ref{eq:prior_ind}) to achieve an optimal posterior concentration rate even when the true sparsity is unknown.

Conditional on $\u$ and $\v,$ we then impose the prior distribution of the loading matrix $\B$ such that
    \begin{equation}
    \label{eq:prior_loading}
    \Pi(\B\in\cB|\u,\v)=\int_{\cB} \prod_{j=1}^p\prod_{k=1}^{q} \sbr{ \delta_0(\beta_{jk})}^{1-u_jv_k}\sbr{\sp_{\Lap(1)}(\beta_{jk})}^{u_jv_k}\d\beta_{jk}
    \end{equation}
for any measurable set $\cB\subset\R^{p\times q}$, where $\delta_0$ denotes the Dirac-delta function at $0$ and $\Lap(1)$ does  the Laplace distribution with scale 1. That is, independently for each loading $\beta_{jk}$, we consider a spike and slab type prior distribution with the Dirac spike and Laplace slab.  The use of the Laplace slab, which is more diffused than the normal distribution, is commonly used in Bayesian sparse factor models \citep{rovckova2016fast,xie2018bayesian,ning2021spike,ohn2020posterior} in order to reduce bias in the estimation of large loadings.

Lastly, we consider the inverse Gamma prior distribution for $\psi$ such that
    \begin{equation}
   \label{eq:prior_noise}
        \Pi(\psi\in\cB')=\P_{\IG(\a)}(\psi\in\cB').
   \end{equation}
for any measurable set $\cB'\subset\R$, where $\IG(\a)$ with $\a:=(a_1,a_2)\in\R_+^2$ denotes the inverse gamma distribution with shape $a_1$ and rate $a_2$.

The proposed prior is carefully designed to attain the posterior consistency of the factor dimensionality adaptively to the sparsity, therefore we call the prior defined through \labelcref{eq:prior_ind}, \labelcref{eq:prior_loading} and \labelcref{eq:prior_noise}  the \textit{adaptive spike and slab (AdaSS)} prior.

\subsection{Posterior computation}
\label{subsec:mcmc}

In this section, we develop an MCMC algorithm to compute the posterior distribution of the parameters $\u:=(u_j)_{j\in[p]}$, $\v:=(v_k)_{k\in[q]}$, $\B:=(\beta_{jk})_{j\in[p],k\in[q]}$, $\psi$ and the latent factors $\Z_1,\dots,\Z_n$. We first introduce additional notations. Let $Y_{ij}$ be the $j$-th element of $\Y_i$ and $Z_{ik}$ be the $k$-th element of $\Z_i$. Let $\cS:=\{j\in[p]:u_j=1\}$ and $\cK:=\{k\in[q]:v_k=1\}$. We use the notation $f(Y|X=x)$ to denote the conditional density of $Y$ given $X=x$. 

To make posterior sampling of the factor loading $\beta_{jk}$ easy, we employ the scale mixture representation of the Laplace distribution. Note that if $\beta_{jk}|\tau_{jk}\sim \N(0,\tau_{jk})$ and $\tau_{jk}\sim \EXP(1/2)$, then marginally we have $\beta_{jk}\sim \Lap(1),$ where $\EXP(1/2)$ stands for the exponential distribution with mean $2$. In the MCMC algorithm, we introduce auxiliary scale parameters $\tau_{jk}$ for $j\in[p]$ and $k\in[q]$. 

Then a single iteration of our proposed MCMC sampler goes as follows.

\null\noindent    
\textbf{Sample $\beta_{jk}$ for $j\in[p]$ and $k\in[q]$:}
We sample   $\beta_{jk}$ from 
    \begin{equation*}
        \beta_{jk}|- \sim
        \begin{cases}
        \N(\hat{\beta}_{jk}, \hat\tau_{jk}) & \mbox{if $u_j=1$ and $v_k=1$}\\
       \delta_0  & \mbox{otherwise },
        \end{cases}
    \end{equation*}
where
    \begin{align}
        \hbeta_{jk}&:= \htau_{jk}\cbr{\psi^{-1}\sum_{i=1}^nZ_{ik}\del[2]{Y_{ij}-\sum_{h\in[q]:h\neq k}Z_{ih}\beta_{jh}}} \label{eq:post_beta}\\
        \htau_{jk}&:=\del{\psi^{-1}\sum_{i=1}^nZ_{ik}^2+\tau_{jk}^{-1}}^{-1}. \label{eq:post_tau}
    \end{align}
  
\null\noindent
\textbf{Sample $\tau_{jk}$ for $j\in[p]$ and $k\in[q]$:}
We sample  $\tau_{jk}$ from 
    \begin{equation}
    \label{eq:sample_tau}
        \tau_{jk}|- \sim
        \begin{cases}
        \GIG(1, \beta_{jk}^2,\frac{1}{2}) & \mbox{if $u_j=1$},\\
        \EXP(\frac{1}{2}) & \mbox{otherwise,}
        \end{cases}
    \end{equation}
where $\GIG(a,b,c)$ denotes the generalized inverse Gaussian distribution with density $\sp_{\GIG(a,b,c)}(z) \propto z^{c-1} \e^{-(az + b/z)/2}\ind(z>0).$  
    
\null\noindent
\textbf{Sample $u_j$ for $j\in[p]$:} 
If $\cS\setminus\{j\}= \emptyset$, we set $u_j=1$.
Otherwise, we sample $u_j$ from $\Ber(\rho_j^{\text{row}}/(1+\rho_j^{\text{row}}))$ with
    \begin{equations}
        \rho_j^{\text{row}}
        &:=\frac{\Pi(u_j=1|-)}{\Pi(u_j=0|-)}\\
        &= \frac{Q_{A}(|\cS\setminus\{j\}|+1,|\cK|)}{Q_{A}(|\cS\setminus\{j\}|, |\cK|)}
        \frac{\prod_{i=1}^n f(\Y_{i}|u_j=1,\v,\T_{j,\cK}, \B_{[-j,:]},\psi, \Z_i)}{\prod_{i=1}^n f(\Y_{i}|u_j=0,\B_{[-j,:]},\psi, \Z_i)}   \\
        &= (p\vee n)^{-A|\cK|} \frac{\binom{p}{|\cS\setminus\{j\}|}}{\binom{p}{|\cS\setminus\{j\}|+1}}
        \frac{\int \prod_{i=1}^n f(\Y_{i}|\B,\psi,\Z_i) \prod_{k=1}^{q} \cbr{\sp_{\N(0,\tau_{jk})}(\beta_{jk})}^{v_k} \cbr{\delta_0 (\beta_{jk})}^{1-v_k}\d\beta_{jk}}{\prod_{i=1}^n f(\Y_{i}|\B_{[j,:]}=\zero, \B_{[-j,:]},\psi,\Z_i)}\\
        & =(p\vee n)^{-A|\cK|}\frac{|\cS\setminus\{j\}|+1}{p-|\cS\setminus\{j\}|}\sqrt{\frac{|\hat\T_{j,\cK}|}{|\T_{j,\cK}|}}\exp\del{\frac{1}{2}\hat\bbeta_{j,\cK}^\top \hat\T_{j,\cK}^{-1}\hat\bbeta_{j,\cK}},
    \end{equations}
where $ \B_{[-j,:]}:=\B_{[[p]\setminus\{j\},:]}$, $\T_{j,\cK}:=\diag((\tau_{jk})_{k\in\cK})$, $\hat\T_{j,\cK}:=\diag((\hat\tau_{jk})_{k\in\cK})$ and $\hat\bbeta_{j,\cK}:=(\hat\beta_{jk})_{k\in\cK}$.

\null\noindent
\textbf{Sample $v_k$ for $k\in[q]$.}
If $\cK\setminus\{k\}= \emptyset$,  we set $v_k=1$. 
Otherwise, we sample $v_k$ from $\Ber(\rho_k^{\text{col}}/(1+\rho_k^{\text{col}}))$ with
    \begin{equations}
        \rho_k^{\text{col}}
        &:=\frac{\Pi(v_k=1|-)}{\Pi(v_k=0| -)}\\
        &=  \frac{Q_{A}(|\cS|,|\cK\setminus\{k\}|+1)}{Q_{A}(|\cS|, |\cK\setminus\{k\}|)} 
        \frac{\prod_{i=1}^n f(\Y_{i}|v_k=1,\u,\T_{\cS,k}, \B_{[:,-k]},\psi, \Z_i)}{\prod_{i=1}^n f(\Y_{i}|v_k=0, \B_{[:,-k]},\psi, \Z_i)}\\
        &= (p\vee n)p^{-A|\cS|} \frac{\binom{q}{|\cK\setminus\{k\}|}} {\binom{q}{|\cK\setminus\{k\}|+1}}
            \frac{\int \prod_{i=1}^n f(\Y_{i}|\B,\psi,\Z_i)\prod_{j=1}^p \cbr{\sp_{\N(0,\tau_{jk})}(\beta_{jk})}^{u_j} \cbr{\delta_0 (\beta_{jk})}^{1-u_j}\d\beta_{jk}}{\prod_{i=1}^n f(\Y_{i}|\B_{[:,k]}=\zero, \B_{[:,-k]},\psi,\Z_i)}\\
        &= (p\vee n)^{-A|\cS|}   \frac{|\cK\setminus\{k\}|+1}{q-|\cK\setminus\{k\}|}
            \sqrt{\frac{|\hat\T_{\cS,k}|}{|\T_{\cS,k}|}}\exp\del{\frac{1}{2}\hat\bbeta_{\cS,k}^\top\hat\T_{\cS,k}^{-1}\hat\bbeta_{\cS,k}},
    \end{equations}
where $\B_{[:,-k]}:=\B_{[:,[q]\setminus\{k\}]}$, $\T_{\cS,k}:=\diag((\tau_{jk})_{j\in\cS})$, $\hat\T_{\cS,k}:=\diag((\hat\tau_{jk})_{j\in\cS})$ and $\hat\bbeta_{\cS,k}:=(\hat\beta_{jk})_{j\in\cS}$.

\null\noindent
\textbf{Sample $\Z_i$ for $i\in[n]$:}
We sample $\Z_i:=(Z_{ik})_{k\in[q]}$  from
    \begin{align*}
        (Z_{ik})_{k\in\cK}|-&\sim \N\del{ \hbalpha_{\cK}^\top\Y_i, \hbXi_{\cK}}\\
         (Z_{ik})_{k\in[q]\setminus\cK}|-&\sim \N\del{ \zero, \I},
    \end{align*}
where
    \begin{align}
        \hbalpha_{\cK}&:= \psi^{-1}\hbXi_{\cK}\B_{[:,\cK]}^\top \label{eq:post_alpha}\\
        \hbXi_{\cK}&:=\del{ \psi^{-1}\B_{[:,\cK]}^\top\B_{[:,\cK]}+\I}^{-1}. \label{eq:post_xi}
    \end{align}

\null\noindent
\textbf{Sample $\psi$:}   
We sample $\psi$  from 
    \begin{equation*}
     \psi|-\sim \IG\del{a_1+ \frac{np}{2}, a_2+\frac{1}{2}\sum_{i=1}^n\sum_{j=1}^p\del[2]{Y_{ij}-\sum_{k\in\cK}Z_{ik}\beta_{jk}}^2}.
    \end{equation*}

We provide several remarks on possible extensions of our Bayesian factor models.

\begin{remark}[Extension to heterogeneous noise variances]
\label{rmk:hetero_noise}
The AdaSS prior can be easily modified for a factor model with heterogeneous noise variances, under which the covariance matrix of the observed variable $\Y_i$ is decomposed as $\var(\Y_i)= \B\B^\top+\bPsi$ with  $\bPsi:=\diag(\psi_1,\dots,\psi_p)$. In this situation, a standard choice of the prior distribution on $\bPsi$ is  a product  of the inverse gamma distributions, that is,
    \begin{equation*}
    \label{eq:prior_noise0}
        \psi_j\indsim \IG(\a_j)
    \end{equation*}
for some $\a_j:=(a_{j1}, a_{j2})\in\R_+^2$ for $j\in[p]$. Then the  conditional posterior of $\psi_j$  is given by
    \begin{equation*}
       \psi_j|-\sim \IG\del{a_{j1}+ \frac{n}{2}, a_{j2}+\frac{1}{2}\sum_{i=1}^n\del{Y_{ij}-\sum_{k\in\cK}Z_{ik}\beta_{jk}}^2}.
    \end{equation*}
For posterior sampling of other parameters, calculations given in \labelcref{eq:post_beta}-\labelcref{eq:post_xi} are modified as  $\hbeta_{jk}:=\htau_{jk}\cbr{\psi_j^{-1}\sum_{i=1}^nZ_{ik}\del{Y_{ij}-\sum_{h\in[]:h\neq k}Z_{ih}\beta_{jh}}}$, $\htau_{jk}:=\del{\psi_j^{-1}\sum_{i=1}^nZ_{ik}^2+\tau_{jk}^{-1}}^{-1}$, $\hbalpha_{\cK}:= \hbXi_{\cK}\B_{[:,\cK]}^\top\bPsi^{-1}$ and $\hbXi_{\cK}:=\del{\B_{[:,\cK]}^\top\bPsi^{-1}\B_{[:,\cK]}+\I}^{-1},$ respectively. 
\end{remark}

\begin{remark}[Extension to correlated factors]
\label{rmk:nondiag}
The factor model we investigated assumes that all components of the latent factor are independent. A more general model would be the correlated factor model such that $\Y_i|(\Z_i=\z_i) \indsim \N(\B\z_i, \psi\I)$ and $\Z_i\iidsim \N(\zero, \bSigma_{\z})$ for some $\bSigma_{\z}\in\bS_{++}^{q},$ and  our AdaSS prior can be easily  modified for this model. If we impose the inverse Wishart prior $\texttt{IW}(\A, \nu)$ with scale matrix $\A\in\bS_{++}^{q}$ and the degrees of freedom $\nu>0$ on $\bSigma_{\z}$, then we can sample $\bSigma_{\z}$ from the conditional posterior
    \begin{align*}
        \bSigma_{\z}|-\sim \texttt{IW}\del[2]{\A+\sum_{i=1}^n\z_i\z_i^\top, \nu+n},
    \end{align*}
while the other posterior sampling schemes remain the same. Unfortunately, the theoretical results in this paper could not be applied directly to the correlated factor model since the correlated factor does not guarantee the required sparsity pattern of the loading matrix to achieve optimal posterior concentration rates.
\end{remark}

\begin{remark}[Post-processing for estimating loading matrices]
The loading matrix $\B$ is not identifiable since for any $q\times q$ orthogonal matrix $\Q$ with $\Q\Q^\top=\I$ the transformed loading matrix $\B\Q$ yields the exact same likelihood as that of $\B$.
Consequently, we need additional effort if we are interested in estimating the loading matrix.  A number of methods have been proposed to resolve this identifiability issue for Bayesian factor analysis. One approach is to impose a prior distribution on the loading matrix satisfying certain identifiability constraints such as the positive diagonal, lower triangular (PLT) constraint \citep{lopes2004bayesian,ghosh2009default,leung2016order,man2022mode}. But as pointed out by \citet{carvalho2008high,assmann2016bayesian}, the posterior distribution obtained under the PLT constraint may not be invariant to the ordering of the observed variables. An alternative approach is to post-process MCMC samples to make $\B$ identifiable \citep{assmann2016bayesian,papastamoulis2022identifiability},
which can be directly applicable to our Bayesian model. We 
illustrate the effectiveness of this approach by analyzing a toy example in \cref{sec:pp}.

\end{remark}

\section{Asymptotic properties of the posterior distribution}
\label{sec:theory}

We study frequentist properties of the posterior distribution induced by the proposed AdaSS prior distribution. Throughout this section, we assume that the number of columns $q$ of the loading matrix $\B$ of our Bayesian model is taken to be sufficiently large so that it is at least the true factor dimensionality. Given data $\Y_{1:n}:=(\Y_1,\dots, \Y_n)$, we denote by $\Pi(\cdot|\Y_{1:n})$ the posterior distribution under the AdaSS prior $\Pi$. Proofs of all results in this section are deferred to \cref{sec:proof} in the supplementary material.


\subsection{Class of covariance matrices}

We first define a  class of matrices to which the true covariance matrix belongs. We denote by $r$ the \textit{true factor dimensionality}. To deal with very high dimensional cases where the dimension $p$ is much larger than the sample size $n$, we impose sparsity on the loading matrix. Specifically, for a loading matrix $\B^\star\in\R^{p\times r}$, we define its (row) support by
    \begin{equation*}
        \supp(\B^\star):=\cbr{j\in[p]:\B^\star_{[j,:]}\neq\zero}.
    \end{equation*}
We say that the loading matrix $\B^\star$ is \textit{$s$-sparse} if $|\supp(\B^\star)|\le s$ and let
    \begin{equation*}
    \cB(p, r,s):=\cbr{\B^\star\in\R^{p\times r}: |\supp(\B^\star)|\le s}
    \end{equation*}
be a set of $p\times r$-dimensional $s$-sparse (loading) matrices.
The parameter space for the covariance matrix we consider throughout the paper is given by
    \begin{equations}
    \label{eq:cov_class}
    \cC(p,r,s,\lambda, \zeta):=\Big\{\bSigma^\star=\B^\star(\B^\star)^\top&+\psi^\star\I:
    \B^\star \in \cB(p, r,s),\lambda_1(\bSigma^\star)\le \lambda,\psi\ge\psi_0, \\
    &\min\cbr[2]{\lambda_r(\B^\star(\B^\star)^\top), \min_{j\in\supp(\B^\star)}\norm[1]{\B^\star\del[1]{\B^\star_{[j,:]}}^\top}_2}\ge\zeta\Big\}
    \end{equations}
for some arbitrarily small constant $\psi_0>0$. We discuss the implications of the conditions determining the class $ \cC(p,r,s,\lambda, \zeta).$
    \begin{itemize}
        \item As we will show in \cref{sec3.2} (\cref{thm:conv_cov}), the posterior concentration rate of the covariance matrix depends on the dimension $p$, sparsity $s$, factor dimensionality $r$ and the upper bound of the largest eigenvalue $\lambda$, but not on $\zeta$, which is needed for the consistency of the factor dimensionality, and $\psi_0$.
        
        \item Our parameter space \labelcref{eq:cov_class} includes loading matrices whose row support is sparse, which is also considered in \cite{cai2013sparse, cai2015optimal,xie2018bayesian,ning2021spike}. On the other hand, \citet{pati2014posterior, gao2015rate, rovckova2016fast, ohn2020posterior} consider the sparsity of the column support, which means that the nonzero entries of each column vector are less than or equal to $s.$ Note that the $s$ column support sparsity implies $sr$ row support sparsity and hence the column and row support sparsities have the same order of $s$ as $s\rightarrow\infty$ when $r$ is bounded.
       
        \item The lower bound $\psi\ge\psi_0,$ which is assumed in \cite{pati2014posterior, ohn2020posterior}, is introduced to avoid ill-conditioned covariance matrices. In contrast, \cite{cai2013sparse, cai2015optimal, gao2015rate,rovckova2016fast,xie2018bayesian,ning2021spike} assume that $\psi$ is fixed in their theories. 
        \item The condition $\lambda_r(\B^\star(\B^\star)^\top)\ge \zeta$ yields the eigengap between spikes and noises, which prevents underestimation of the factor dimensionality. The same condition is assumed by \cite{cai2013sparse, cai2015optimal, gao2015rate, ohn2020posterior}. 
        \item Note that $\norm[1]{\B^\star\del[1]{\B^\star_{[j,:]}}^\top}_2\ge \norm[1]{\B^\star_{[j,:]}}_2^2$, thus the condition $\min_{j\in\supp(\B)}\norm[1]{\B^\star\del[1]{\B^\star_{[j,:]}}^\top}_2\ge \zeta$ is met when the magnitudes of the nonzero rows do not vanish too quickly. This condition enables  accurate estimation of the sparsity level of the true loading matrix. This condition is similar to the beta-min condition (see Section 7.4 of \citet{buhlmann2011statistics})  in high-dimensional sparse regression models. We believe that this condition is indispensable unless the column vectors are assumed to be orthogonal  \citep{cai2013sparse, cai2015optimal, gao2015rate} or the true sparsity level $s$ is known \citep{ohn2020posterior}.
    \end{itemize}

Note that we allow the model architecture parameters  $p$, $s$, $r$ $\lambda$ and $\zeta$ to depend on $n$, but we do not specify the subscript $n$ to those quantities, e.g., keep using $p$ instead of $p_n$, for notational simplicity.

\subsection{Posterior consistency of the factor dimensionality}\label{sec3.2}

In this section, we explore asymptotic properties of the posterior distribution of the factor dimensionality in the sparse factor model. For a loading matrix $\B\in\R^{p\times q}$, we define the factor dimensionality corresponding to $\B$ as
   \begin{equation}
        \xi:=\xi(\B):=\abs{\cbr{k\in[q]:\B_{[:,k]}\neq\zero}},
    \end{equation}
that is, the factor dimensionality $\xi(\B)$ is equal to the number of nonzero columns of $\B$.
The following theorem  shows that the posterior distribution of the factor dimensionality behaves nicely. 

\begin{theorem}  
\label{thm:conv_nfac}
Assume that $r\le p/2$, $\lambda\lesssim s$ and $\epsilon:=\lambda\sqrt{sr\log (p\vee n)/n}=o(1)$. Moreover, assume that $\zeta \ge c_0\epsilon$ for sufficiently large $c_0>0$. Then for any $\delta>0$, there exists a constant $A_\delta>0$ depending only on $\delta$ such that 
    \begin{align}
    \label{eq:conv_nfac_div}
     \inf_{\bSigma^\star\in\cC(p, r, s, \lambda, \zeta)}\P_{\bSigma^\star}^{(n)}\sbr{\Pi\del{r\le \xi(\B)\le (1+\delta) r|\Y_{1:n}}}  \to 1 
    \end{align}
for the  prior distribution defined through \labelcref{eq:prior_ind}, \labelcref{eq:prior_loading} and \labelcref{eq:prior_noise} with $A>A_\delta$, and $\a\in\R_+^2.$ Furthermore, if $r\lesssim \log (p\vee n)/\log n$, then there exists an universal constant $A_0>0$ such that 
    \begin{align}
    \label{eq:conv_nfac_bdd}
     \inf_{\bSigma^\star\in\cC(p, r, s, \lambda, \zeta)}\P_{\bSigma^\star}^{(n)}\sbr{\Pi\del{ \xi(\B)= r|\Y_{1:n}}}  \to 1,
    \end{align}
whenever $A>A_0.$
\end{theorem}

\cref{thm:conv_nfac} implies that a posteriori the factor dimensionality $\xi$ is not much larger than the true one $r$, and $\xi$ concentrates on $r$  asymptotically if $r$ is not too large.
.0
\citet{gao2015rate} attains posterior consistency of the factor dimensionality of the orthogonal loading matrix under a mildly growing regime $r\lesssim \log p$. Our condition $r\lesssim \log ( p\vee n)/\log n$ for the posterior consistency is slightly stronger than \cite{gao2015rate}. However, note that we do not impose the orthogonality constraint on the loading matrix which makes posterior computation difficult.

\citet{ohn2020posterior} obtains  posterior consistency without any condition on the growth rate of the true factor dimensionality $r$ by using a prior that strongly regularizes $\xi.$ However, this strong regularization sacrifices the convergence rate of the covariance matrix by a factor of $\sqrt{s}$ compared to the optimal rate. Another critical drawback of their prior is that the knowledge of the true sparsity level  is required to select the hyper-parameters in the prior. In contrast, the AdaSS prior attains the posterior consistency of the factor dimensionality \textit{without knowing  the true sparsity level}.
 
On the frequentist side, \citet{cai2013sparse, cai2015optimal} proposed consistent estimators of the factor dimensionality for sparse factor models. However, \citet{cai2013sparse} requires a $\sqrt{s}$ times larger detection rate for the eigengap 
(i.e., the lower bound of $\zeta$) than ours and \citet{cai2015optimal} is nonadaptive to the true sparsity. Moreover, a known and fixed noise variance $\psi$ is required for the consistency of both estimators, while our consistency result is \textit{adaptive to the unknown noise level}.

\begin{remark}
\label{remark:q_choice}
One should set $q>r$ to correctly estimate $r,$ but $r$ is unknown. 
A naive strategy would be to set $q$ very large, e.g., $q=p-1$, so that $q\ge r$. However, unnecessarily large $q$ requires huge computation. A better strategy for choosing $q$ is to set $q=\sqrt{n}.$ This choice is based on our  posterior contraction rate $\lambda\sqrt{sr\log(p\vee n)/n}=o(1)$ of the covariance matrix given in \cref{thm:conv_cov} in the next section. Since $r=\rank(\B^\star(\B^\star)^\top)\le s$ for the true loading matrix $\B^\star$, if $\lambda \log (p\vee n)\gtrsim1$ we have $r=o(\sqrt{n})$. Therefore, asymptotically, the upper bound $q=\sqrt{n}$ does not underdetermine the true factor dimensionality. A similar problem occurs in \cite{gao2015rate}, where the authors assume $r\lesssim \log p$ and set $q\asymp p^{b}$ for some $b>0$ so that $q\gg r$.  When $p$ is large, our choice $q=\sqrt{n}$ is much smaller than that of \cite{gao2015rate}, which leads to more efficient computation
\end{remark}

\begin{remark}
In \cref{thm:conv_nfac}, we assume an upper bound for the largest eigenvalue such that $\lambda\lesssim s$. This bound is mild in view of the random matrix theory. Suppose that $\tilde{\B}\in\R^{s\times r}$ is a  random matrix whose entries are independent centered random variables with finite fourth moments. Then by Theorem 2 of \cite{latala2005some}, since $r\le s$, we have $\E\norm[0]{\tilde{\B}}\lesssim \sqrt{s}+\sqrt{r}\lesssim \sqrt{s}$. Therefore, $\E[\lambda_1(\tilde{\B}\tilde{\B}^\top)]=\E\norm[0]{\tilde{\B}\tilde{\B}^\top}\lesssim s$. \citet{pati2014posterior} and \citet{rovckova2016fast} assumed the same condition as ours, while other studies on the Bayesian covariance estimation \citep{gao2015rate, xie2018bayesian} used a stronger condition that the largest eigenvalue of the true covariance matrix is bounded.
\end{remark}
 
In \cref{thm:conv_nfac}, we show that the true factor dimensionality is almost consistently recovered whenever the eigengap $\zeta$ is larger than the \textit{detection rate}  $\lambda\sqrt{sr\log p/n}$ by a sufficiently large constant $c_0>0.$  As shown in the next proposition, this detection rate is optimal when $r\lesssim 1$ in the sense that any method cannot consistently estimate the factor dimensionality when the eigengap $\zeta$ is less than $a_0\lambda \sqrt{ s\log (p/s)/n}$ for some constant $a_0>0$. This result is an extension of Theorem 5 of \citet{cai2015optimal} for unknown $\psi$ and  diverging $\lambda.$ 

\begin{proposition}
\label{prop:nfac_lower_bound}
Assume that $s\log p/n=o(1)$. Then there exists a constant $a_0>0$ such that if $0\le\zeta\le a_0\lambda \sqrt{ s\log (p/s)/n}$,
    \begin{equation}
        \inf_{\hat{r}:\R^{p\times n}\mapsto \bN}\sup_{\bSigma\in\cC(p, r,s, \lambda, \zeta)}\P^{(n)}_{\bSigma}\del{\hat{r}(\Y_{1:n})\neq r}\ge \frac{1}{4}
    \end{equation}
for all but finitely many $n,$ where the infimum runs over all possible estimator $\hat{r}$ of $r.$
\end{proposition}

\subsection{Posterior concentration rate of the covariance matrix}

In the linear factor model, the covariance matrix $\B\B^\top + \psi\I$ determines its distribution. In this section, we prove that the posterior distribution of the covariance matrix in our Bayesian model concentrates  around the true covariance matrix at a near-optimal rate, which is summarized in the next theorem.

\begin{theorem} 
\label{thm:conv_cov}
Assume that $\lambda\lesssim s$. Then there exists a constant $M>0$ such that 
    \begin{align}
    \label{eq:conv_cov}
     \sup_{\bSigma^\star\in \cC(p, r, s, \lambda,0)}\P_{\bSigma^\star}^{(n)}\sbr{\Pi\del{\norm{\bSigma-\bSigma^\star} >M\lambda\sqrt{\frac{ sr\log (p\vee n)}{n}}\Big|\Y_{1:n}}} &\to0
    \end{align}
for the  prior distribution defined through \labelcref{eq:prior_ind}, \labelcref{eq:prior_loading} and \labelcref{eq:prior_noise} with $A>0$ and $\a\in\R_+^2.$
\end{theorem}


Note that the lower bound of the eigengap $\zeta$ is set to 0 in \cref{thm:conv_cov} while it should be larger than a certain rate in \cref{thm:conv_nfac}, that is, the eigengap condition is required 
only for consistent estimation of the factor dimensionality but not  for the contraction of the covariance matrix. This difference implies that the optimal estimation of the covariance matrix does not require the consistent estimation of the factor dimensionality.

Our posterior concentration rate in \labelcref{eq:conv_cov}  is near optimal when $r\lesssim 1$ as shown in the following proposition, which is a direct consequence of Theorem 5.4 of \citet{pati2014posterior}.

\begin{proposition}
\label{prop:cov_lower_bound}
Assume that $s\log p/n=o(1)$ and $r\lesssim 1$. Then 
    \begin{equation}
        \inf_{\hat{\bSigma}:\R^{p\times n}\mapsto \bS_{++}^p}  \sup_{\bSigma^\star\in \cC(p, r, s, \lambda,0)}\P_{\bSigma^\star}^{(n)}\sbr[2]{\norm[1]{\hat{\bSigma}-\bSigma^\star}}
        \gtrsim \lambda\sqrt{\frac{ s\log(p/s)}{n}},
    \end{equation}
for all but finitely many $n,$  where the infimum runs over all possible estimator $\hat{\bSigma}$ of $\bSigma^\star.$
\end{proposition}


\section{Numerical examples}
\label{sec:numerical}

We evaluate the empirical performance of the proposed Bayesian model with the AdaSS prior through simulation studies and real data analysis. For each posterior computation, we run the MCMC sampler described in \cref{subsec:mcmc} for 3,000 iterations discarding the first 500 as burn-in, and by thinning every 5, we obtain the final 500 MCMC samples from the posterior. We give the convergence diagnostics via trace, autocorrelation and partial autocorrelation plots of some randomly selected parameters in \cref{sec:conv} in the supplementary material, which confirm that the MCMC sampler converges well.

\subsection{Simulation study}
\label{subsec:simulation}

In this section, we conduct an extensive numerical study to compare the performance of the AdaSS prior for estimating the factor dimensionality and the covariance matrix  with various competitors. Throughout the simulation study, we set the number of columns of the loading matrix  $q=\ceil{\sqrt{n}}$ for a sample size $n$ and the hyperparameters $A=0.1$ and $\a=(0.01, 0.01)$.

\subsubsection{Posterior distribution of the factor dimensionality}

We first compare the AdaSS prior and the spike and slab  with the two-parameter IBP prior of \cite{ohn2020posterior} for evaluating the concentration behaviors of their posterior distributions of the factor dimensionality.  We only consider the prior of \cite{ohn2020posterior} since other Bayesian models either do not infer the factor dimensionality \citep{ning2021spike,xie2018bayesian} or do not achieve posterior consistency of the factor dimensionality \citep{rovckova2016fast,bhattacharya2011sparse, srivastava2017expandable} or are purely theoretical (i.e., do not have a posterior computation algorithm) \citep{gao2015rate}.

There are two hyperparameters of the spike and slab with the two-parameter IBP prior of \cite{ohn2020posterior}. The first hyperparameter denoted by $\alpha$ controls the factor dimensionality and the second hyperparameter denoted by $\kappa$ does the sparsity of the loading matrix. For $\kappa$, we choose $p^{1.1}$  as recommended by \cite{ohn2020posterior}. For $\alpha$, we consider three values: $p^{-30}$, $p^{-25}$ and $p^{-20}$.  \cite{ohn2020posterior} proved that using $\alpha=p^{-As}$ for a constant $A>0$ and the true sparsity $s$ can lead to the posterior consistency.
Assuming that $s=30,$ these choices of $\alpha$ correspond to the choices of $A\in \{4/6, 5/6,1\}.$ We use the MCMC sampler used in  \cite{knowles2011nonparametric, ohn2020posterior} for approximating the posterior.

We generate a data set consisting of $n$ synthetic random vectors from the multivariate normal distribution with mean $\zero$ and variance $\B^\star(\B^\star)^\top+2\I$ independently, where  $\B^\star$ is a $s$-sparse $p\times r$ loading matrix. For the true loading matrix, we randomly select the location of $s$ nonzero rows and sample the elements in the nonzero rows from $\{-2,2\}$ randomly. We set $(n,p)=(100,1000)$ and let the sparsity $s$ and factor dimensionality $r$ vary among $s\in\{10, 30, 50\}$ and $r\in\{1,3,5\}$, respectively.

\cref{fig:nfactor} presents the posterior distribution of the factor dimensionality for the AdaSS prior and spike and slab with the two-parameter IBP prior with different $\alpha$.  The posterior distribution under the AdaSS prior concentrates at the true factor dimensionality quite well for all nine cases, while the performance of the two-parameter IBP prior depends heavily on the choice of the hyperparameter $\alpha$. If $\alpha$ is not sufficiently small, the resulting posterior distribution apparently overestimates the true factor dimensionality. The smallest choice of $\alpha=p^{-30}$ estimates the factor dimensionality consistently for some cases, but also severe underestimation occurs in other cases. The results of this simulation show that there is no universally good choice of the hyperparameter $\alpha$ in the two-parameter IBP across different levels of the sparsity, while the AdaSS prior performs consistently well with a single choice of the hyperparameter.

\begin{figure}
    \centering
    \makebox[\textwidth][c]{
    \includegraphics[scale=0.15]{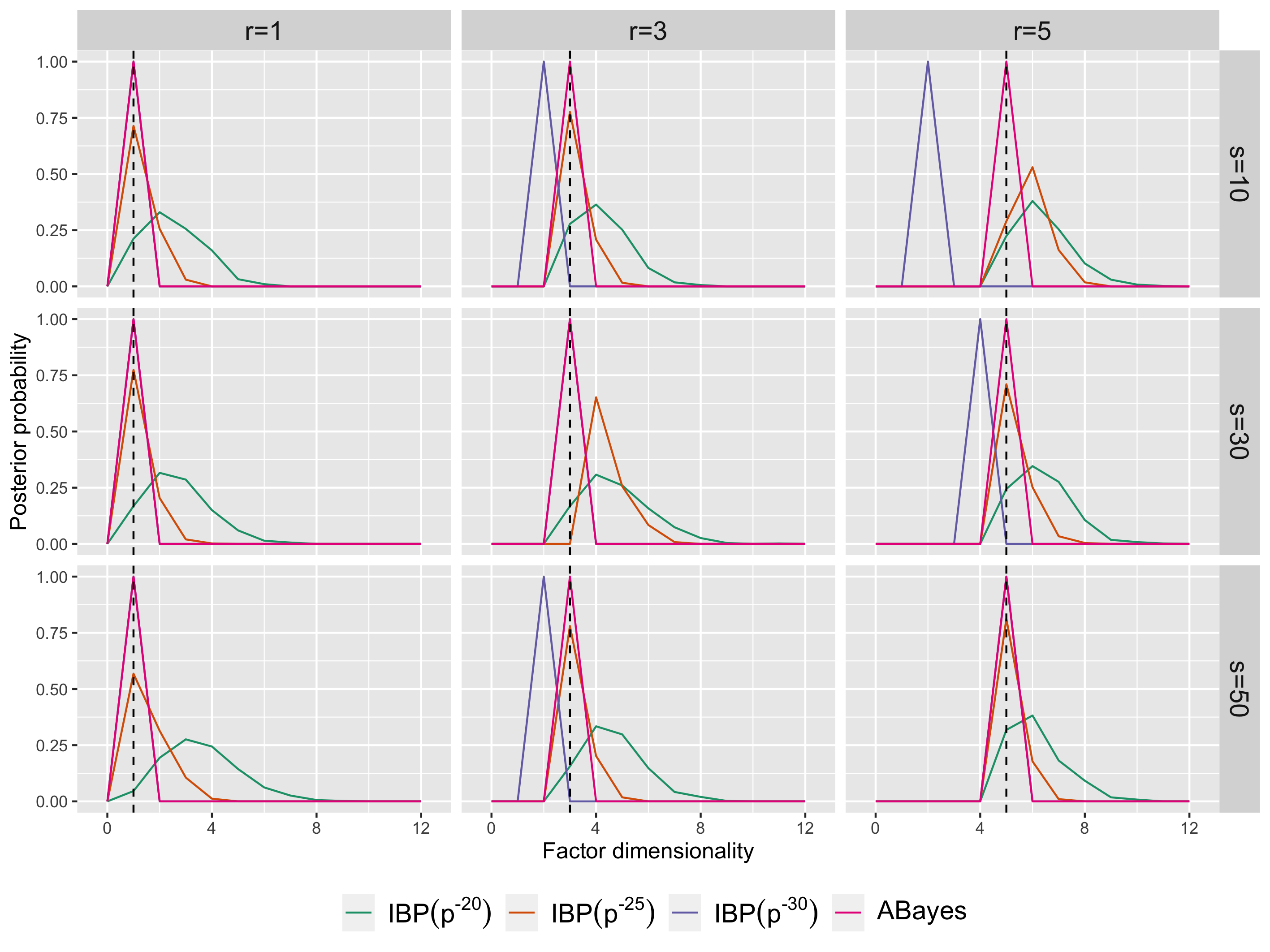}
    }
    \caption{Posterior distributions of the factor dimensionality for the proposed adaptive Bayesian (AdaSS) prior and spike and slab with the two-parameter IBP prior with various $\alpha$ (IBP($\alpha$)). The black dashed vertical lines indicate the true factor dimensionality $r$.}
    \label{fig:nfactor}
\end{figure}

\subsubsection{Comparison with frequentist estimators} 

In this simulation, we compare the performance of the AdaSS prior with some frequentist estimators for point estimation of the factor dimensionality. For our Bayesian model, we use the mode of the posterior distribution of the factor dimensionality as a point estimator. We consider the following five frequentist estimators as competitors: with $\S$ and  $\mathbf{R}$ being the sample covariance and correlation matrices, respectively, and $r_{\max}\in[p]$ pre-specified, 
    \begin{itemize}
        \item Eigenvalue threshold estimator (ET, \cite{onatski2010determining})
            \begin{equation*}
                \hat{r}_{\textsc{ET}}:=\sum_{j=1}^p\ind\del{\lambda_j(\S)>w\lambda_{r_{\max}+1}(\S)+(1-w)\lambda_{2r_{\max}+1}(\S)},
            \end{equation*}
        where $w:=2^{2/3}(2^{2/3}-1).$
        \item Eigenvalue ratio estimator (ER, \cite{ahn2013eigenvalue,lam2012factor}):
            \begin{equation*}
                \hat{r}_{\textsc{ER}}:=\argmax_{j\in[r_{\max}]}\lambda_j(\S)/\lambda_{j+1}(\S).
            \end{equation*}
        \item Growth ratio estimator (GR, \cite{ahn2013eigenvalue}):
            \begin{equation*}
                \hat{r}_{\textsc{GR}}:=\argmax_{j\in[r_{\max}]}\frac{\log(L_{j-1}/L_{j})}{\log(L_{j}/L_{j+1})},
            \end{equation*}
            where $L_j:=\sum_{t=j+1}^p\lambda_t(\S)$.
        \item Adjusted correlation threshold estimator (ACT, \cite{fan2020estimating}):
            \begin{equation*}
                \hat{r}_{\textsc{ACT}}:=\max\cbr{j\in[r_{\max}]:\lambda_j^\dag>1+\sqrt{p/(n-1)}},
            \end{equation*}
            where
                \begin{equation*}
                    \lambda_j^\dag:=\sbr{\frac{1-w_j}{\hat\lambda_j}-\frac{w_j}{p-j}\cbr{\sum_{t=j+1}^p(\hat\lambda_t-\hat\lambda_j)^{-1}+ 4(\hat\lambda_{j+1}-\hat\lambda_j)^{-1}}}^{-1}
                \end{equation*}
            with $\hat\lambda_j:=\lambda_j(\mathbf{R})$ and $w_j:=(p-j)/(n-1)$ for $j\in[p].$
    \item Diagonal thresholding estimator (DT, \cite{cai2013sparse}):
                \begin{equation*}
                \hat{r}_{\textsc{DT}}:=\max\cbr{j\in[r_{\max}]:\lambda_j(\check\S_{[J, J]})>2\del{1 + \sqrt{|J|/n}+\sqrt{(2(1+|J|)\log(\e p) + 6\log n)/n}}^2},
            \end{equation*}
        where $\check\S:=n^{-1}\sum_{i=1}^n(\Y_i+\Z^0_i)(\Y_i+\Z^0_i)^\top$ with $\Z^0_1,\dots, \Z^0_n\iidsim \N(\zero, \I)$ and 
            \begin{align*}
                J:=\cbr{j\in[p]:\check\S_{[j,j]}\ge 2+6\sqrt{\log p/n}}.
            \end{align*}
    \end{itemize}
We fix $r_{\max}=10$ throughout the simulation study.

A synthetic data set of size $n$ is generated from the $p$-dimensional normal distribution with mean $\zero$ and covariance $\B^\star(\B^\star)^\top+\I$, where the true loading matrix $\B^\star\in\R^{p\times r}$ is $s$-sparse. The true loading matrix $\B^\star$ is generated as follows: we first select $s$ nonzero rows and sample the elements in the nonzero rows from the uniform distribution on $[-4/\sqrt{s},-3/\sqrt{s}]\cup[3/\sqrt{s},4/\sqrt{s}]$. We take two sample sizes $n=50$ and $n=100$. We fix the dimension $p=1,000$ and let the sparsity $s$ and factor dimensionality $r$ vary among  $s\in\{10, 30, 50\}$ and $r\in\{1,3,5\}$, respectively. 

The simulation results based on 100 synthetic data sets of size $n=50$ and $n=100$ are summarized in Tables \labelcref{tab:nfac_c1} and  \labelcref{tab:nfac_c2}, respectively. We see that the proposed Bayesian model with the AdaSS prior outperforms the other competitors in the estimation of the factor dimensionality. The AdaSS prior has the highest proportion of the correct estimation for 17 out of the total 18 simulation setups. In particular, there are considerable performance gaps between the AdaSS prior and the  other competitors when sparsity is small ($s=10$) or the factor dimensionality is large ($r=5$).

\begin{table}[]
\centering
\caption{Proportions of correct estimation (``True''), overestimation (``Over'') and underestimation (``Under'') of the estimated factor dimensionalities for various sparsity $s$ and true factor dimensionality $r$ obtained on 100 synthetic data sets of size $n=50$. ``Ave'' is the average of the estimated factor dimensionalities.}
\label{tab:nfac_c1}
\begin{tabular}{ccccccccc}
\hline
 $s$ & $r$ &   & ET & ER & GR & ACT & DT & AdaSS  \\\hline
 \multirow{12}{*}{10} & \multirow{4}{*}{1} & True & 85 & 97 & 97 & 96 & 100 & 100 \\ 
   &  & Over & 14 & 3 & 3 & 4 & 0 & 0 \\ 
   &  & Under & 1 & 0 & 0 & 0 & 0 & 0 \\ 
   &  & Ave & 1.14 & 1.03 & 1.03 & 1.04 & 1 & 1 \\ 
  \cline{2-9} & \multirow{4}{*}{3} & True & 24 & 3 & 3 & 0 & 0 & 97 \\ 
   &  & Over & 1 & 0 & 0 & 0 & 0 & 0 \\ 
   &  & Under & 75 & 97 & 97 & 100 & 100 & 3 \\ 
   &  & Ave & 2.16 & 1.35 & 1.37 & 1.01 & 1 & 2.97 \\ 
  \cline{2-9} & \multirow{4}{*}{5} & True & 0 & 0 & 0 & 0 & 0 & 77 \\ 
   &  & Over & 0 & 0 & 0 & 0 & 0 & 15 \\ 
   &  & Under & 100 & 100 & 100 & 100 & 100 & 8 \\ 
   &  & Ave & 2.97 & 1.54 & 1.55 & 1.01 & 1.1 & 5.06 \\ 
   \hline
\multirow{12}{*}{30} & \multirow{4}{*}{1} & True & 86 & 96 & 97 & 97 & 100 & 100 \\ 
   &  & Over & 12 & 4 & 3 & 3 & 0 & 0 \\ 
   &  & Under & 2 & 0 & 0 & 0 & 0 & 0 \\ 
   &  & Ave & 1.1 & 1.1 & 1.07 & 1.03 & 1 & 1 \\ 
  \cline{2-9} & \multirow{4}{*}{3} & True & 46 & 15 & 15 & 0 & 0 & 94 \\ 
   &  & Over & 0 & 0 & 0 & 0 & 0 & 0 \\ 
   &  & Under & 54 & 85 & 85 & 100 & 100 & 6 \\ 
   &  & Ave & 2.43 & 1.57 & 1.57 & 1.09 & 1 & 2.93 \\ 
  \cline{2-9} & \multirow{4}{*}{5} & True & 0 & 1 & 1 & 0 & 0 & 87 \\ 
   &  & Over & 0 & 0 & 0 & 0 & 0 & 0 \\ 
   &  & Under & 100 & 99 & 99 & 100 & 100 & 13 \\ 
   &  & Ave & 3.2 & 1.72 & 1.77 & 1.06 & 1 & 4.86 \\ 
   \hline
\multirow{12}{*}{50} & \multirow{4}{*}{1} & True & 80 & 98 & 97 & 95 & 100 & 100 \\ 
   &  & Over & 17 & 2 & 3 & 5 & 0 & 0 \\ 
   &  & Under & 3 & 0 & 0 & 0 & 0 & 0 \\ 
   &  & Ave & 1.17 & 1.04 & 1.05 & 1.05 & 1 & 1 \\ 
  \cline{2-9} & \multirow{4}{*}{3} & True & 58 & 9 & 10 & 6 & 0 & 73 \\ 
   &  & Over & 2 & 0 & 0 & 0 & 0 & 0 \\ 
   &  & Under & 40 & 91 & 90 & 94 & 100 & 27 \\ 
   &  & Ave & 2.62 & 1.52 & 1.56 & 1.4 & 1 & 2.67 \\ 
  \cline{2-9} & \multirow{4}{*}{5} & True & 2 & 2 & 2 & 0 & 0 & 61 \\ 
   &  & Over & 0 & 0 & 0 & 0 & 0 & 0 \\ 
   &  & Under & 98 & 98 & 98 & 100 & 100 & 39 \\ 
   &  & Ave & 3.39 & 1.94 & 1.99 & 1.25 & 1 & 4.54 \\ 
   \hline
\end{tabular}
\end{table}

\begin{table}[]
\centering
\caption{Proportions of correct estimation (``True''), overestimation (``Over'') and underestimation (``Under'') of the estimated factor dimensionalities for various sparsity $s$ and true factor dimensionality $r$ obtained on 100 synthetic data sets of size $n=100$. ``Ave'' is the average of the estimated factor dimensionalities.}
\label{tab:nfac_c2}
\begin{tabular}{ccccccccc}
\hline
 $s$ & $r$ &   & ET & ER & GR & ACT & DT & AdaSS \\ \hline
 \multirow{12}{*}{10} & \multirow{4}{*}{1} & True & 78 & 100 & 100 & 93 & 100 & 100 \\ 
   &  & Over & 22 & 0 & 0 & 7 & 0 & 0 \\ 
   &  & Under & 0 & 0 & 0 & 0 & 0 & 0 \\ 
   &  & Ave & 1.24 & 1 & 1 & 1.07 & 1 & 1 \\ 
  \cline{2-9} & \multirow{4}{*}{3} & True & 66 & 9 & 10 & 0 & 0 & 86 \\ 
   &  & Over & 2 & 0 & 0 & 0 & 0 & 13 \\ 
   &  & Under & 32 & 91 & 90 & 100 & 100 & 1 \\ 
   &  & Ave & 2.7 & 1.55 & 1.57 & 1.02 & 1.25 & 3.12 \\ 
  \cline{2-9} & \multirow{4}{*}{5} & True & 6 & 0 & 0 & 0 & 0 & 20 \\ 
   &  & Over & 0 & 0 & 0 & 0 & 0 & 77 \\ 
   &  & Under & 94 & 100 & 100 & 100 & 100 & 3 \\ 
   &  & Ave & 3.63 & 1.83 & 1.86 & 1 & 2.13 & 5.93 \\ 
   \hline
\multirow{12}{*}{30} & \multirow{4}{*}{1} & True & 83 & 100 & 100 & 93 & 100 & 100 \\ 
   &  & Over & 17 & 0 & 0 & 7 & 0 & 0 \\ 
   &  & Under & 0 & 0 & 0 & 0 & 0 & 0 \\ 
   &  & Ave & 1.18 & 1 & 1 & 1.09 & 1 & 1 \\ 
  \cline{2-9} & \multirow{4}{*}{3} & True & 92 & 33 & 34 & 5 & 0 & 99 \\ 
   &  & Over & 4 & 0 & 0 & 2 & 0 & 1 \\ 
   &  & Under & 4 & 67 & 66 & 93 & 100 & 0 \\ 
   &  & Ave & 3 & 1.96 & 1.98 & 1.56 & 1 & 3.01 \\ 
  \cline{2-9} & \multirow{4}{*}{5} & True & 46 & 3 & 4 & 0 & 0 & 68 \\ 
   &  & Over & 1 & 0 & 0 & 0 & 0 & 32 \\ 
   &  & Under & 53 & 97 & 96 & 100 & 100 & 0 \\ 
   &  & Ave & 4.45 & 2 & 2.1 & 1.43 & 1 & 5.32 \\ 
   \hline
\multirow{12}{*}{50} & \multirow{4}{*}{1} & True & 85 & 100 & 100 & 94 & 100 & 100 \\ 
   &  & Over & 15 & 0 & 0 & 6 & 0 & 0 \\ 
   &  & Under & 0 & 0 & 0 & 0 & 0 & 0 \\ 
   &  & Ave & 1.16 & 1 & 1 & 1.06 & 1 & 1 \\ 
  \cline{2-9} & \multirow{4}{*}{3} & True & 99 & 59 & 61 & 64 & 0 & 98 \\ 
   &  & Over & 1 & 0 & 0 & 2 & 0 & 0 \\ 
   &  & Under & 0 & 41 & 39 & 34 & 100 & 2 \\ 
   &  & Ave & 3.01 & 2.36 & 2.38 & 2.67 & 1 & 2.98 \\ 
  \cline{2-9} & \multirow{4}{*}{5} & True & 68 & 15 & 20 & 2 & 0 & 91 \\ 
   &  & Over & 0 & 0 & 0 & 0 & 0 & 6 \\ 
   &  & Under & 32 & 85 & 80 & 98 & 100 & 3 \\ 
   &  & Ave & 4.68 & 2.36 & 2.65 & 2.61 & 1 & 5.03 \\ 
   \hline
\end{tabular}
\end{table}

\subsubsection{Covariance matrix estimation} 

In this simulation study, we compare the AdaSS prior with  other competitors  for covariance matrix estimation. For competitors, we consider the principal orthogonal complement thresholding method (POET, \cite{fan2013large}), the variational inference method for Bayesian sparse PCA (SPCA-VI, \cite{ning2021spike}), the Bayesian sparse factor models with  multiplicative gamma process shrinkage prior (MGPS, \cite{bhattacharya2011sparse}) and two maximum a posteriori estimators that employ the multi-scale generalized double Pareto prior (MDP, \citep{srivastava2017expandable}) and the spike-and-slab lasso with Indian buffet process prior (SSL-IBP, \citep{rovckova2016fast}), respectively. For the POET and SPCA-VI, the factor dimensionality must be selected in advance and we use the true factor dimensionality for this. We use the posterior mean of the covariance matrix as the point estimator for the MGPS and  AdaSS priors. 

We generate  100 synthetic data sets with sample size $n=50$ and $n=100$, respectively, and we report the averages of the scaled spectral norm losses $\norm[0]{\hat\bSigma-\bSigma^\star}/\|\bSigma^\star\|$ between the point estimate $\hat\bSigma$ of each estimator and the true covariance matrix $\bSigma^\star$ obtained over 100 synthetic data sets in Tables \ref{tab:cov_c1} and \ref{tab:cov_c2}. The AdaSS prior performs generally well, while the POET, MGPS and MDP are significantly inferior. SSL-IBP is not much worse and performs best for the setups with $s=50.$

\begin{table}[ht]
\centering
\caption{The averages and standard errors of the scaled spectral norm losses of the estimators of the covariance matrix obtained on 100 synthetic data sets with $n=50$.} 
\label{tab:cov_c1}
\begin{tabular}{cccccccc} \hline
 $s$ & $r$ & POET & SPCA-VI & MGPS & MDP & SSL-IBP & ABayes \\  \hline
 \multirow{3}{*}{10} & 1 & 2.366 (0.151) & 0.245 (0.108) & 2.591 (2.04) & 2.103 (0.135) & 0.689 (0.062) & \textbf{0.233 (0.103)} \\ 
   & 3 & 1.83 (0.254) & 0.398 (0.12) & 1.45 (0.863) & 1.583 (0.212) & 0.646 (0.094) & \textbf{0.301 (0.113)} \\ 
   & 5 & 1.59 (0.232) & 0.422 (0.111) & 1.195 (0.575) & 1.271 (0.191) & 0.696 (0.101) & \textbf{0.335 (0.107)} \\ 
   \hline
\multirow{3}{*}{30} & 1 & 2.375 (0.155) & 0.772 (0.102) & 1.945 (1.292) & 2.104 (0.14) & 0.699 (0.067) & \textbf{0.624 (0.152)} \\ 
   & 3 & 2.073 (0.202) & 0.674 (0.117) & 2.078 (1.309) & 1.839 (0.184) & 0.696 (0.063) & \textbf{0.609 (0.172)} \\ 
   & 5 & 1.868 (0.192) & 0.644 (0.086) & 1.551 (0.87) & 1.649 (0.175) & 0.684 (0.052) & \textbf{0.631 (0.138)} \\ 
   \hline
\multirow{3}{*}{50} & 1 & 2.345 (0.146) & 0.901 (0.039) & 2.018 (1.52) & 2.072 (0.13) & \textbf{0.759 (0.065)} & 0.847 (0.134) \\ 
   & 3 & 2.145 (0.194) & 0.762 (0.102) & 1.996 (1.168) & 1.901 (0.175) & \textbf{0.695 (0.069)} & 0.966 (0.208) \\ 
   & 5 & 2.013 (0.2) & 0.744 (0.098) & 1.519 (0.752) & 1.786 (0.182) & \textbf{0.709 (0.067)} & 1.049 (0.228) \\ 
   \hline
\end{tabular}
\end{table}

\begin{table}[ht]
\centering
\caption{The averages and standard errors of the scaled spectral norm losses of the estimators of the covariance matrix obtained on 100 synthetic data sets with $n=100$.} 
\label{tab:cov_c2}
\begin{tabular}{cccccccc} \hline
 $s$ & $r$ & POET & SPCA-VI & MGPS & MDP & SSL-IBP & ABayes \\  \hline
 \multirow{3}{*}{10} & 1 & 1.437 (0.108) & \textbf{0.164 (0.064)} & 2.374 (1.123) & 1.322 (0.101) & 0.523 (0.048) & 0.17 (0.066) \\ 
   & 3 & 1.136 (0.146) & 0.274 (0.074) & 1.565 (1.129) & 1.044 (0.141) & 0.483 (0.066) & \textbf{0.218 (0.074)} \\ 
   & 5 & 0.999 (0.151) & 0.281 (0.069) & 1.051 (0.407) & 0.881 (0.135) & 0.464 (0.057) & \textbf{0.244 (0.075)} \\ 
   \hline
\multirow{3}{*}{30} & 1 & 1.437 (0.09) & 0.366 (0.105) & 2.489 (1.124) & 1.317 (0.085) & 0.534 (0.053) & \textbf{0.317 (0.088)} \\ 
   & 3 & 1.28 (0.138) & 0.449 (0.083) & 2.187 (0.946) & 1.175 (0.133) & 0.533 (0.055) & \textbf{0.353 (0.105)} \\ 
   & 5 & 1.178 (0.116) & 0.457 (0.079) & 1.723 (1.618) & 1.084 (0.109) & 0.538 (0.059) & \textbf{0.4 (0.087)} \\ 
   \hline
\multirow{3}{*}{50} & 1 & 1.428 (0.095) & 0.707 (0.112) & 2.357 (1.226) & 1.306 (0.09) & 0.57 (0.065) & \textbf{0.536 (0.133)} \\ 
   & 3 & 1.328 (0.102) & 0.628 (0.102) & 2.236 (1.072) & 1.219 (0.1) & \textbf{0.547 (0.058)} & 0.587 (0.133) \\ 
   & 5 & 1.261 (0.119) & 0.57 (0.072) & 1.729 (0.97) & 1.161 (0.112) & \textbf{0.564 (0.059)} & 0.652 (0.123) \\ 
   \hline
\end{tabular}
\end{table}

\subsection{Real data analysis}
\label{subsec:real_data}

In this section, we analyze gene expression data on aging in mice from the AGEMAP (Atlas of Gene Expression in Mouse Aging Project) database \cite{zahn2007agemap}. We obtained this data from the online website \href{http://statweb.stanford.edu/~owen/data/AGEMAP/}{http://statweb.stanford.edu/$\sim$owen/data/AGEMAP}. There are 5 female and 5 male mice in each age group, where there are 4 age groups of 1, 6, 16 and 24 months. Thus there  are $40$ mice in total.  From each of 40 mice, 16 microarrays obtained from 16 different tissues were prepared, and from each microarray, gene expression levels of $8,932$ probes were measured. In this paper,  we focus only on the microarray data from the cerebrum tissue, for which the rotation test of \cite{perry2010rotation} provided strong evidence for the presence of the latent factor. We will call this one tissue data set with sample size $n=40$ and dimension $p=8,932$ the AGEMAP data for simplicity.

We preprocessed the AGEMAP data following the regression model of \cite{perry2010rotation}. We obtained the mean-centered data by regressing out an intercept, sex, and age effects on each of the 8,932 outcomes. Then the factor dimensionality is estimated based on the mean-centered data set. We consider the factor model with heterogeneous noise variances and impose the AdaSS prior presented in \cref{rmk:hetero_noise}. We set $q=10\ge \sqrt{n}= \sqrt{40}$, $A=0.1$ and $\a_j=(0.01, 0.01)$ for every $j\in[p]$ in the prior. Then we take the posterior mode of the factor dimensionality as the point estimate. For comparison, we also considered the five frequentist methods described in \cref{subsec:simulation}, i.e., ET, ER, GR, ACT and DT.

\cref{tab:agemap_nfac} provides the factor dimensionality estimates by the proposed Bayesian model and the five competing frequentist methods. The four methods including the AdaSS prior estimate the factor dimensionality by 1. The presence of the one-dimensional latent factor was advocated by the rotation test of \cite{perry2010rotation}.

\cref{fig:agemap_hist_z} shows the histogram of the posterior means of the latent factors $\E(Z_{ik^*}|\Y_{1:n})$ for $i\in[n]$ obtained under the AdaSS prior, where $k^*$ denotes the index of the nonzero column of the loading matrix, i.e., $\B_{[:,k^*]}\neq\zero$ under the posterior distribution. The bimodality of the histogram is clearly shown, which is also confirmed by \cite{perry2010rotation}.  \cref{fig:agemap_hist_spar} presents the posterior distribution of the sparsity $|\supp(\B)|$ of the loading matrix, which ranges from 79.4\% to 82\%. A similar 78\% sparsity of the estimated factor model was reported by \cite{rovckova2016fast}. 

\begin{table}[t]
\centering
\caption{Estimated factor dimensionality for the  AGEMAP data.}
\label{tab:agemap_nfac}
\begin{tabular}{cccccc}
  \hline
ET & ER & GR & ACT & DT & AdaSS \\ 
  \hline
  8 &   1 &   1 &  10 &   1 &   1 \\ 
   \hline
\end{tabular}
\end{table}

\begin{figure}[t]
    \centering   
    \begin{subfigure}[c]{0.45\textwidth}
        \includegraphics[scale=0.115]{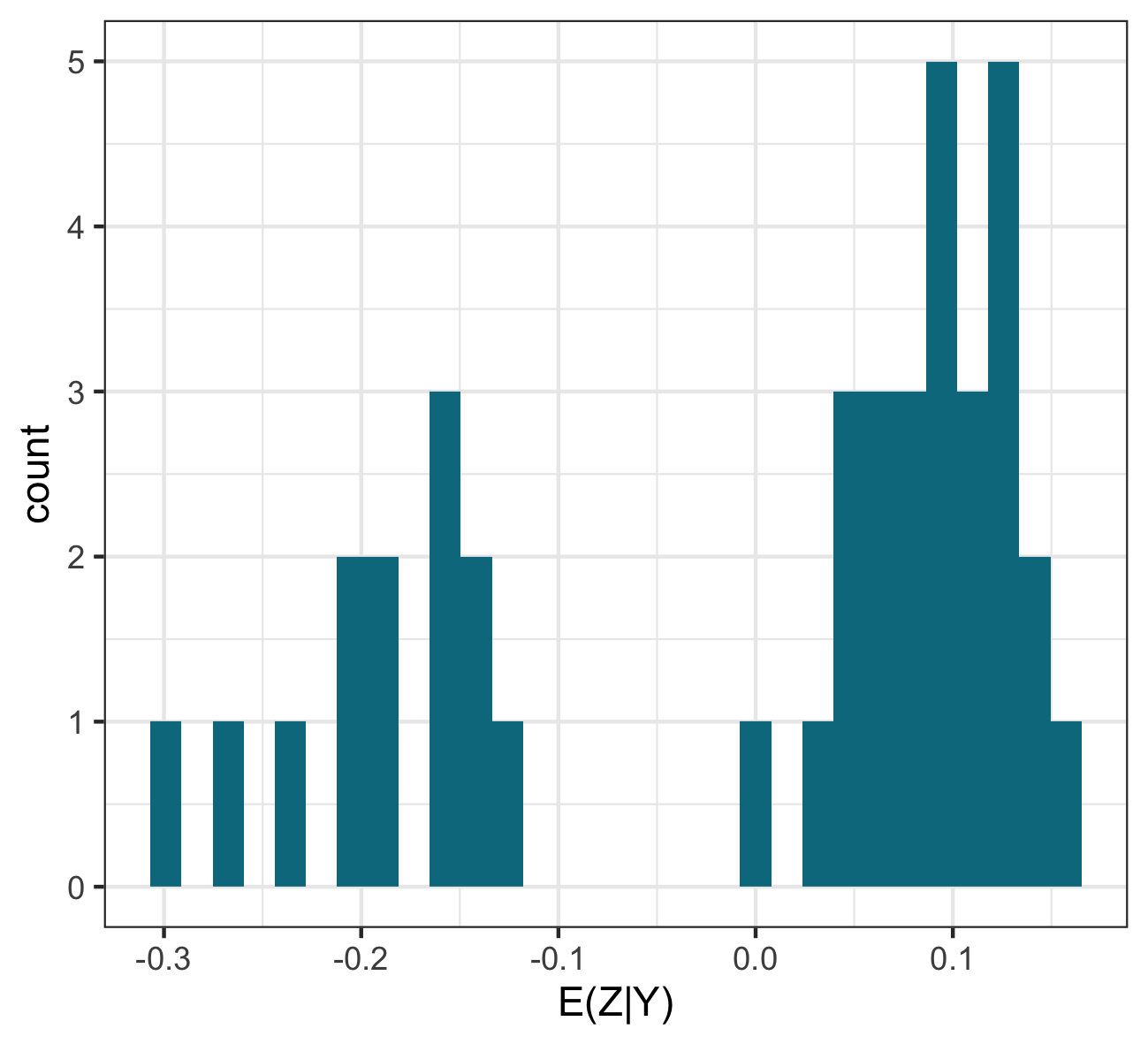}
        \subcaption{}
        \label{fig:agemap_hist_z}
        \end{subfigure}\quad
    \begin{subfigure}[c]{0.45\textwidth}
        \includegraphics[scale=0.115]{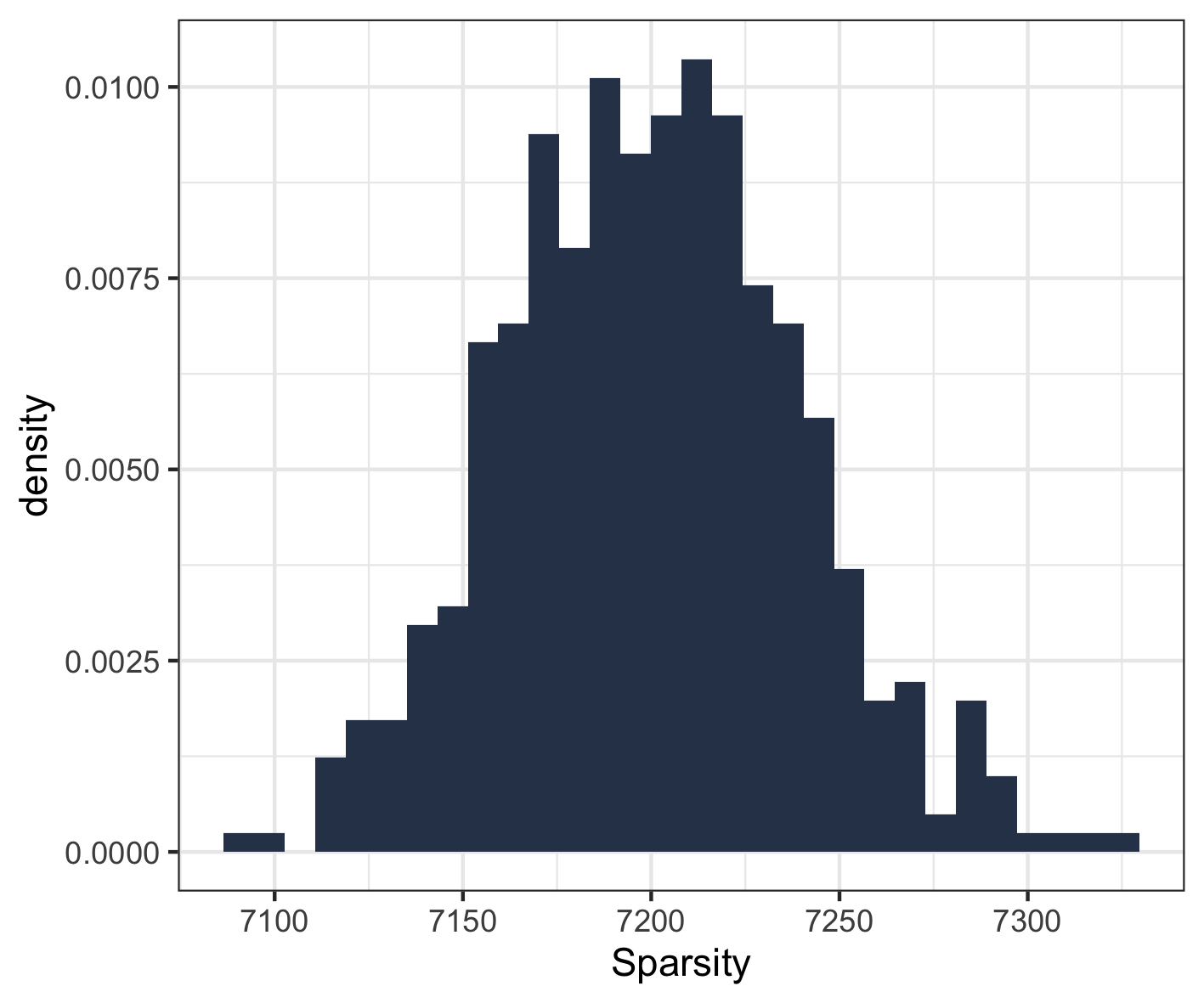}
        \subcaption{}
        \label{fig:agemap_hist_spar}
        \end{subfigure}
\caption{(a) Histogram of the posterior means of the latent factors for each observation; (b) Posterior distribution of the sparsity of the loading matrix for the AGEMAP data.}
    \label{fig:agemap_hist}
\end{figure}

\section{Concluding remarks}
\label{sec:conclusion}

In this paper, we proposed a novel prior distribution, called AdaSS, to infer high-dimensional sparse factor models. We proved that the resulting posterior distribution asymptotically concentrates at the true factor dimensionality without knowing the sparsity level of the true loading matrix. 
This adaptivity to the sparsity is a remarkable advantage of the proposed method over other theoretically consistent estimators such as the point estimator in \cite{cai2015optimal} and Bayesian posterior distribution in \cite{ohn2020posterior}. We also showed that the proposed model attained the optimal detection rate of the eigengap for consistent estimation of the factor dimensionality. Moreover, the concentration rate of the posterior distribution of the covariance matrix is optimal when the true factor dimensionality is bounded and equal or faster than those of other Bayesian models otherwise. Our numerical studies amply confirmed the theoretical results and  provided strong empirical support to the proposed AdaSS prior.

With our prior, nonzero row vectors of the loading matrix $\B$ are not sparse. That is, when $s'$-many  nonzero row vectors and the factor dimensionality are given, all the entries of the corresponding $s' \times \xi$ sub-matrix are all nonzero. In practice, one may want to have sparsity in nonzero row vectors. Our prior can be modified easily to ensure such sparsity without hampering the asymptotic properties, which we will report somewhere else.

There are several promising directions for future work. In this paper, we consider the \textit{static} factor model where the observations are assumed to be identically distributed. However, this static factor model may be inadequate to capture the dependence structure in some types of data, e.g., time series data. As an alternative, we may consider a dynamic factor model, where the covariance matrix as well as the factor dimensionality can be different at each time point. It would of interest to study the posterior consistency of the factor dimensionality which possibly varies over time. Another promising avenue of research is to develop the Bayesian factor model which deals  with non-Gaussian or mixed-type observed variables. We believe that the proposed Bayesian model can be easily extended to those types of data using the Gaussian copula factor model developed by \cite{murray2013bayesian}. It would be interesting to investigate the theoretical properties of such a non-Gaussian extension of the proposed Bayesian model.

\bibliographystyle{plainnat}
\bibliography{_references}

\clearpage

\newpage

\setcounter{page}{1}
\renewcommand{\thepage}{S-\arabic{page}}

\crefalias{section}{appendix}

	\begin{center}
		{\Large \textsc{Supplement to ``A Bayesian sparse factor model with adaptive posterior concentration''}} \\
		\medskip 
		{Ilsang Ohn, Lizhen Lin and Yongdai Kim}
		\medskip
	\end{center}

In this supplementary material, we provide proofs of the main results and additional numerical results.

\appendix

\section{Proofs}
\label{sec:proof}

Before giving the proofs, we introduce some additional notations. We define
    \begin{align*}
        \bar\cC(p,q):=\cbr{\bSigma=\B\B^\top+\psi\I:\B\in\R^{p\times q}, \psi>0}
    \end{align*}
for $q\in\bN$. 
Note that the set $\bar\cC(p,q)$ includes the support of the AdaSS prior.

\subsection{Proofs of \cref{thm:conv_nfac,thm:conv_cov}}
\label{sec:proof:main}

Note that
    \begin{align*}
        \P_{\bSigma^\star}^{(n)}\Pi(\cU|\Y_{1:n})
        =\P_{\bSigma^\star}^{(n)}\frac{N_n(\cU)}{D_n}
    \end{align*}
for any measurable subset $\cU\subset\bar\cC(p)$, where we denote
    \begin{align*}
        N_n(\cU)&:=\int_{\cU}\frac{\sp_{\bSigma}^{(n)}}{\sp_{\bSigma^\star}^{(n)}}(\Y_{1:n})\Pi(\d\bSigma),\\
        D_n&:=\int\frac{\sp_{\bSigma}^{(n)}}{\sp_{\bSigma^\star}^{(n)}}(\Y_{1:n})\Pi(\d\bSigma).
    \end{align*}
We will prove \cref{thm:conv_nfac,thm:conv_cov} by deriving a lower bound of $D_n$ and an upper bound
of $N_n(\cU)$ for $\cU$ related to Theorems.

The following lemma provides a high-probability lower bound of $D_n$. The proof is deferred to \cref{subsec:proof_d_n} in the supplementary material.

\begin{lemma}
\label{lem:d_n}
Suppose that the prior distribution is defined through \labelcref{eq:prior_ind}, \labelcref{eq:prior_loading} and \labelcref{eq:prior_noise} with $A>0$ and $\a\in\R_+^2$. Assume that  $\lambda\lesssim s$. Then for any $\bSigma^\star\in\cC(p, r, s, \lambda, 0)$, we have 
    \begin{equation}
        \P_{\bSigma^\star}^{(n)}\del{D_n\ge \e^{-C_1s(r\log n\vee\log p) - Asr\log (p\vee n)}}\ge 1-\frac{1}{\log n}
    \end{equation}
for some universal constant $C_1>0$.
\end{lemma}

We now prove  \cref{thm:conv_nfac,thm:conv_cov}. Since the proof of \cref{thm:conv_nfac} uses the result of \cref{thm:conv_cov}, we first prove \cref{thm:conv_cov}.

\subsubsection{Proof of \cref{thm:conv_cov}}

\begin{proof}[Proof of \cref{thm:conv_cov}]
Fix $\bSigma^\star=\B^\star(\B^\star)^\top+\psi^\star\I\in\cC(p, r, s, \lambda, 0)$. Recall that we let $\epsilon:=\lambda\sqrt{sr\log  (p\vee n)/n}$. By \cref{lem:d_n}, there exists an universal constant $C_1>0$  such that $\P^{(n)}_{\bSigma^\star}(\fD_n)\ge 1-\log^{-1}n$ where
    \begin{equation}
    \label{eq:set_denom}
        \fD_n:=\cbr{\Y_{1:n}\in\R^{p\times n}:D_n\ge \e^{-C_1s(r\log n\vee\log p) - Asr\log  (p\vee n)}}.
    \end{equation}
Note that
    \begin{equation*}
        \Pi(|\supp(\B)|>t)= \Pi(\|\u\|_0>t)\lesssim \sum_{t=\ceil{t}}^{p} \sum_{k=1}^{\infty}  (p\vee n)^{-Atk}\lesssim \e^{-At\log (p\vee n)}
    \end{equation*}
for any $t\in \mathbbm{N}.$
Hence, for a sufficiently large constant $C_2>0$, we have
    \begin{align*}
        \P_{\bSigma^\star}^{(n)}\sbr{\Pi\del{|\supp(\B)|>C_2sr|\Y_{1:n}}}
        &\le \P^{(n)}_{\bSigma^\star}(\fD_n^c)
         +\frac{\Pi\del{|\supp(\B)|>C_2sr}}{ \e^{-C_1s(r\log n\vee\log p) - Asr\log (p\vee n)}}\\
        &\le \log^{-1}n + \e^{C_1s(r\log n\vee\log p)-A(C_2-1)sr\log (p\vee n)}
        \to0.
    \end{align*}
Therefore it suffices to show that $\P_{\bSigma^\star}^{(n)}\sbr{\Pi(\cU|\Y_{1:n})}$ converges to zero, where
  \begin{equation*}
        \cU:=\cbr{\bSigma\in\bar\cC(p,q):|\supp(\B)|\le C_2sr, \norm[0]{\bSigma-\bSigma^\star}\ge M\epsilon}.
    \end{equation*}
We use the test function given in the next lemma to bound the posterior probability of $\cU$.  The proof of the lemma is provided in \cref{subsec:proof_test}  in the supplementary material.

\begin{lemma}
\label{lem:test}
Let $\bSigma^\star:=\B^\star(\B^\star)^\top+\psi^\star\I\in\cC(p,r,s,\lambda,0)$ and define a set
    \begin{equation*}
        H_1:=\cbr{\bSigma\in\bar\cC(p,q): |\supp(\B)|\le t,  \norm[0]{\bSigma-\bSigma^\star}\ge M\tepsilon},
    \end{equation*}
where $t\in\bN$, $q\in\bN$,  $M>2^{4/3}$ and $\tepsilon>0$. Then there exists a test function $\phi:\R^{p\times n}\mapsto[0,1]$ such that
   \begin{align*}
    \P^{(n)}_{\bSigma^\star}\phi(\Y_{1:n})&\le 3\exp\del{2t\log p+(C_3+1)(t+s)-\frac{C_3M^{1/2}}{(1+\psi_0^{-2})\lambda^2}n\tepsilon^2}, \\
    \sup_{\bSigma:\bSigma\in H_1} \P^{(n)}_{\bSigma}(1-\phi(\Y_{1:n})) &\le \exp\del{C_3(t+s)-\frac{C_3M}{4}n\tepsilon^2}
    \end{align*}
for some universal constant $C_3>0$.
\end{lemma}

For the test function $\phi$ given in \cref{lem:test} and the set $\fD_n$ defined in \labelcref{eq:set_denom}, we have
    \begin{align*}
    \P^{(n)}_{\bSigma^\star}\sbr{\Pi\del{\mathcal{U}|\Y_{1:n}}}
    &\le \P^{(n)}_{\bSigma^\star}\phi +\P^{(n)}_{\bSigma^\star}\sbr{(1-\phi)\Pi\del{\cU|\Y_{1:n}}\ind(\fD_n)} + \P_{\bSigma^\star}(\fD_n^c) \\
    &\le 3\e^{C_4sr\log p -C_5M^{1/2}n\epsilon^2/\lambda} + \e^{C_6sr\log (p\vee n) -C_7Mn\epsilon^2} +  \log^{-1}n  \\
    &\le 3 \e^{(C_8-C_{9}M^{1/2})sr\log (p\vee n)}+  \log^{-1}n
     \end{align*}
for some universal positive constants $C_4,\dots, C_{9}$. Hence, we obtain the desired result by choosing sufficiently large $M>0$.
\end{proof}

\subsubsection{Proof of \cref{thm:conv_nfac}}
\label{subsec:proof_nfac}

\begin{proof}[Proof of \labelcref{eq:conv_nfac_div}]
Fix $\bSigma^\star = \B^\star(\B^\star)^\top+\psi^\star\I\in\cC(p, r, s, \lambda, \zeta)$.  We first show that the expectation of the posterior probability of the event $\cbr{\bSigma\in\bar{\cC}(p,q):\xi(\B)<r}$ converges to zero. converges to zero. For this, it suffices to show that
    \begin{equation}
    \label{eq:rank_def}
        \cbr{\bSigma\in\bar{\cC}(p,q):\xi(\B)<r}\subset\cbr{\bSigma\in\bar{\cC}(p,q):\norm{\bSigma-\bSigma^\star} >\zeta/2}
    \end{equation}
due to \cref{thm:conv_cov}.
Suppose that $\xi(\B)<r$. Since $\B^\star$ is of full rank,  there exists $\w_1\in \Span(\B)^\perp \cap \Span(\B^\star)$ with $\norm{\w_1}_2=1$. For such $\w_1$, we have
	\begin{align*}
	\norm{\B\B^\top-\B^\star(\B^\star)^\top} 
	&\ge \norm{\del{\B\B^\top-\B^\star(\B^\star)^\top}\w_1}_2 \\
	&= \norm{\B^\star(\B^\star)^\top\w_1}_2\ge \zeta.
	\end{align*}
Let $\cJ:=\Span(\B)\cup \Span(\B^\star)$. Since $\rank(\cJ)< 2r\le p$, there exists  $\w_2\in \cJ^\perp$  with $\norm{\w_2}_2=1$. Then $\norm{\bSigma-\bSigma^\star}\ge \norm{\del{\bSigma-\bSigma^\star}\w_2}_2=\abs{\psi-\psi^\star}$. By the triangular inequality, we obtain
	\begin{equation}
	\label{eq:tri_ineq}
	\norm{\B\B^\top-\B^\star(\B^\star)^\top}\le \norm{\bSigma-\bSigma^\star}+\norm[0]{\del[0]{\psi-\psi^\star}\I} \le  2\norm{\bSigma-\bSigma^\star},
\end{equation}
which proves \labelcref{eq:rank_def}. By the assumption that $\zeta\ge c_0\epsilon$ for sufficiently large $c_0>0$, it follows that $\P_{\bSigma^\star}^{(n)}\Pi(\xi(\B)<r|\Y_{1:n})\to 0$ from \cref{thm:conv_cov}.

For the event  $\cbr{\bSigma\in\bar{\cC}(p,q):\xi(\B)>(1+\delta)r}$, we use the bound
    \begin{equation}
        \P_{\bSigma^\star}^{(n)}\Pi\del{\xi(\B)>(1+\delta)r|\Y_{1:n}}\le T_1+ T_2,
    \end{equation}
where we denote
    \begin{align*}
        T_1 &:=\P_{\bSigma^\star}^{(n)}\Pi\del{\xi(\B)>(1+\delta)r, |\supp(\B)|\ge s|\Y_{1:n}}\\
        T_2 &:=\P_{\bSigma^\star}^{(n)}\Pi\del{|\supp(\B)|< s|\Y_{1:n}}.
    \end{align*}
For $T_1$, note that
    \begin{align*}
     \Pi\del{\xi(\B)>(1+\delta)r, |\supp(\B)|\ge s} 
     &\le \Pi(\|\u\|_0\ge s,  \|\v\|_0>(1+\delta)r)\\
    &\lesssim \sum_{k\in[q]:k>(1+\delta)r}\sum_{t\in[p]:t\ge s} \e^{-Atk\log (p\vee n)}\\
    &\lesssim \e^{-A(1+\delta)sr\log (p\vee n)}.
    \end{align*}
Hence, by \cref{lem:d_n}, there exist constants $C_1>0$ and $C_2>0$ such that 
    \begin{equations}
    \label{eq:t1_bound}
        T_1&\le \P^{(n)}_{\bSigma^\star}(\fD_n^c) + \P_{\bSigma^\star}^{(n)}\sbr{\Pi\del{\xi(\B)>(1+\delta)r, |\supp(\B)|\ge s|\Y_{1:n}}\ind(\fD_n)}\\
        &\le\log^{-1} n + \frac{\Pi\del{\xi(\B)>(1+\delta)r, |\supp(\B)|\ge s} }{\exp\del{C_1s(r\log n\vee\log p) + Asr\log (p\vee n)}}\\
        &\le \log^{-1} n + \exp\del{C_2sr\log (p\vee n) - A\delta sr\log (p\vee n)},
    \end{equations}
where $\fD_n:=\cbr{\Y_{1:n}\in\R^{p\times n}:D_n\ge \e^{-C_1s(r\log n\vee\log p) - Asr\log (p\vee n)}}.$
Therefore, $T_1$ converges to zero if $A>C_2/\delta$. 
For $T_2$, we will show that
    \begin{equation}
    \label{eq:supp_def}
        \cbr{\bSigma\in\bar{\cC}(p,q):|\supp(\B)|< s}\subset\cbr{\bSigma\in\bar{\cC}(p,q):\norm{\bSigma-\bSigma^\star} >\zeta/2}.
    \end{equation}
Suppose that $|\supp(\B)|<s$. Then $\cS^-:=\supp(\B^\star)\setminus \supp(\B)$ is not empty, and so we can find $j^*\in[p]$ such that $j^*\in\cS^-$. Let $\w_3:=(v_{3j})_{j\in[p]}$ be a vector such that $w_{3j^*}=1$ and $w_{3j}=0$ for $j\neq j^*$. Then since $\B_{[j^*,:]}=\zero$, we have
    \begin{align*}
    \norm{\B\B^\top-\B^\star(\B^\star)^\top} 
    &\ge \norm{(\B\B^\top-\B^\star(\B^\star)^\top)\w_3}_2 \\
    & = \norm[1]{\B^\star\del[1]{\B^\star_{[j^*,:]}}^\top}_2\ge\zeta
    \end{align*}
which,  together with \labelcref{eq:tri_ineq}, proves \labelcref{eq:supp_def}. Thus \cref{thm:conv_cov} with the assumption  $\zeta\ge c_0\epsilon$ implies that $T_2=o(1),$ which completes the proof.
\end{proof}

\begin{proof}[Proof of  \labelcref{eq:conv_nfac_bdd}]
It is sufficient to prove that 
    \begin{equation*}
        T_1^*:=\P_{\bSigma^\star}^{(n)}\Pi\del{\xi(\B)>r, |\supp(\B)|\ge s|\Y_{1:n}}\to0.
    \end{equation*}
Using a similar argument used in \labelcref{eq:t1_bound}, if $r\lesssim  \log (p\vee n)/\log n$, we have that 
    \begin{align*}
        T_1^*
        &\le \log^{-1} n + \exp\del{C_1s(r\log n\vee\log p)  + Asr\log (p\vee n)-As(r+1)\log (p\vee n)}\\
        &\le \log^{-1} n + \exp\del{C_3s\log (p\vee n) - As\log(p\vee n)},
    \end{align*}
for some $C_3>0$. Thus if $A>C_3$, we obtain the desired result.
\end{proof}

\subsection{Proof of \cref{prop:nfac_lower_bound}}

\begin{proof}[Proof of \cref{prop:nfac_lower_bound}]
Let $\cW(p,s):=\cbr{\w\in\{0,1\}^{p}:\|\w\|_0=s, w_p=0}$. For a given $\w\in\cW(p,s)$, we define
    \begin{equation*}
        \bSigma(\w):=\frac{\lambda}{2}\I +\frac{\lambda}{2}\mathbf{e}_p\mathbf{e}_p^\top + \frac{\zeta}{s}\w\w^\top,
    \end{equation*}
where $\mathbf{e}_p$ is the vector of zeros except the  $p$-th element which is equal to 1 and $\zeta>0$ is a constant such that $ \zeta\le \lambda/2$. Then, $\bSigma(\w)\in C(p,2,s,\lambda, \zeta)$. Thus, letting $\bSigma_0:=\frac{\lambda}{2}(\I+\mathbf{e}_p\mathbf{e}_p^\top)$ and
    \begin{equation*}
        \cH(p, s, \lambda, \zeta) :=\cbr{\bSigma(\w): \w\in\cW(p,s)},
    \end{equation*}
we have
    \begin{align*}
         &\inf_{\hat{r}:\R^{p\times n}\mapsto \bN}\sup_{\bSigma\in\cC(p, r, s, \lambda, \zeta)}\P^{(n)}_{\bSigma}\del{\hat{r}(\Y_{1:n})\neq r}\\
         &\ge E_{p,s,\lambda,\zeta}:= \inf_{\phi:\R^{p\times n}\mapsto \{0,1\}}
        \sbr{\P^{(n)}_{\bSigma_0}\del{\phi(\Y_{1:n})=1}+ \sup_{\bSigma\in\cH(p, s, \lambda, \zeta)} \P^{(n)}_{\bSigma}\del{\phi(\Y_{1:n})=0}}.
    \end{align*}
To obtain a lower bound of $E_{p,s,\lambda,\zeta}$, we follow a standard argument of minimax theory. Let $\rho:=2\zeta/(s\lambda)$.
Since $|\bSigma(\w)|=\lambda\del{\frac{\lambda}{2}+\zeta}\del{\frac{\lambda}{2}}^{p-2}$ and $|\bSigma_0|=\lambda\del{\frac{\lambda}{2}}^{p-1},$ the Kullback-Leibler (KL) divergence from $\sp_{\bSigma_0}^{(n)}$ to $\sp_{\bSigma(\w)}^{(n)}$ is given by
    \begin{align*}
    \kl(\sp_{\bSigma(\w)}^{(n)}, \sp_{\bSigma_0}^{(n)})
        &=n\int \log (\frac{\sp_{\bSigma(\w)}(\y)}{\sp_{\bSigma_0}(\y)} )\sp_{\bSigma(\w)}(\y)\d\y\\
        &=\frac{n}{2}\tr\del{\bSigma(\w)\bSigma_0^{-1}-\I}-\frac{n}{2}\log \abs[1]{\bSigma(\w)\bSigma_0^{-1}}\\
        &=\frac{n}{2}\del{\tr\del{\rho\w\w^\top}-\log\frac{\lambda/2+\zeta}{\lambda/2}}\\
        &=\frac{n}{2}\del{\rho s-\log(1+\rho s)}\\
        &\le \alpha:=\frac{n}{4}\rho^2s^2=n\zeta^2/\lambda^2
    \end{align*}
for any $\w\in\cW(p,s)$. Moreover, we have
    \begin{equation*}
        M:=|\cH(p, s, \lambda, \zeta)|=\binom{p-1}{s-1}\ge \del{\frac{p-1}{s-1}}^{s-1}\ge\del{\frac{p}{s}}^{s-1}\ge\del{\frac{p}{s}}^{s/2}.
    \end{equation*}
Thus, by Proposition 2.3 of \cite{tsybakov2008introduction} with $\tau=1/M$,
    \begin{align*}
        E_{p,s,\lambda,\zeta}&\ge \frac{1}{2}\del{1-\frac{\alpha + \sqrt{\alpha/2}}{\log M}}\\
        &\ge \frac{1}{2}\del{1-(2+\sqrt{2})\frac{\alpha\vee1}{s\log (p/s)}}.
    \end{align*}
We take $\zeta:=a\lambda \sqrt{s\log (p/s)/n}$ with $a\le a_0:=(4+2\sqrt{2})^{-1/2}$ so that $\alpha\le a_0^2s\log(p/s)\le  s\log(p/s)/(4+2\sqrt{2}) $ for sufficiently large $p$. Also, $\zeta\le \lambda /2$ for all sufficiently large $n$ since $s\log p=o(n)$. Hence $E_{p,s,\lambda,\zeta}$ is further bounded below by $1/4$, which is the desired result.
\end{proof}

\section{Proofs of lemmas in \cref{sec:proof}}
\label{subsec:proof_lemmas}

\subsection{Proof of \cref{lem:d_n}}
\label{subsec:proof_d_n}

We first introduce two lemmas that will be used in the proofs.

\begin{lemma}
\label{lem:d_n0}
 Let $\Pi$ be a distribution on $\bS_{++}^p$ and $\bSigma^\star\in\bS_{++}^p$. For $n$ iid $p$-dimensional random variables $\Y_{1:n}:=(\Y_1,\dots, \Y_n)$,  let us denote $D_n:=\int\frac{\sp_{\bSigma}^{(n)}}{\sp_{\bSigma^\star}^{(n)}}(\Y_{1:n})\Pi(\d\bSigma)$. Then for any $\eta\in(0, \lambda_p(\bSigma^\star)\wedge 1)$, 
	\begin{align*}
	\P^{(n)}_{\bSigma^\star}\del{D_n\ge \exp\del{-C_1 \frac{n\eta^2\log \rho(\bSigma^\star)}{\lambda_{p}^2(\bSigma^\star)}\vee \frac{\eta\sqrt{n\log n}}{\lambda_{p}(\bSigma^\star)}}\Pi\del{\fnorm{\bSigma-\bSigma^\star} <\eta}}\ge 1-\frac{1}{\log n}
	\end{align*}
for some universal constant $C_1>0$, where $\rho(\bSigma^\star):=2\lambda_1(\bSigma^\star)/\lambda_p(\bSigma^\star)$
\end{lemma}

\begin{proof}
Fix $\eta\in(0, \lambda_p(\bSigma^\star)\wedge 1)$ and let $\cE:=\cbr{\fnorm{\bSigma-\bSigma^\star} <\eta}$. Let $\Pi(\cdot|\cE):=\Pi(\cdot\cap\cE)/\Pi(\cE)$, which is the re-normalized restriction of $\Pi$ on $\cE.$ 
By Jensen's inequality,
    \begin{equation*}
        \log \del{\frac{D_n}{\Pi(\cE)}}
        \ge \int_{\cE} \cbr{\frac{n}{2}\log |\bSigma^\star\bSigma^{-1}|-\frac{1}{2}\sum_{i=1}^n\Y_i^\top(\bSigma^{-1}-(\bSigma^\star)^{-1})\Y_i } \Pi(\d\bSigma|\cE).
    \end{equation*}
Let $Q_i:=\Y_i^\top(\bSigma^{-1}-(\bSigma^\star)^{-1})\Y_i$, $W_i:=Q_i-\P_{\bSigma^\star}Q_i$ and $\bar{W}:=n^{-1} \sum_{i=1}^nW_i$. Note that $\P_{\bSigma^\star}Q_i=\tr(\bSigma^\star\bSigma^{-1}-\I)$ and that, by Supplement Lemma 1.3 of \cite{pati2014posterior}, for any $\eta\in(0,1)$ such that $\fnorm{\bSigma-\bSigma^\star}\le \eta$ and $\eta<2\lambda_p(\bSigma^\star)$, we have
    \begin{equation*}
        \log(|\bSigma^\star\bSigma^{-1}|)-\tr(\bSigma^\star\bSigma^{-1}-\I)\ge -C_2\frac{\eta^2\log \rho(\bSigma^\star)}{\lambda_p^2(\bSigma^\star)}
    \end{equation*}
for some constant $C_2>0$. Therefore, 
    \begin{align*}
        \log \del{\frac{D_n}{\Pi(\cE)}}
        &\ge \int_{\cE} \cbr{\frac{n}{2}\log |\bSigma^\star\bSigma^{-1}|-\tr(\bSigma^\star\bSigma^{-1}-\I) -\bar{W} } \Pi(\d\bSigma|\cE)\\
        &\ge -C_2\frac{n\eta^2\log \rho(\bSigma^\star)}{2\lambda_p^2(\bSigma^\star)}
        - \frac{n}{2}\int \bar{W} \Pi(\d\bSigma|\cE).
    \end{align*}
It remains to derive the high-probability upper bound of $|\bar{W}|$. By the variance formula for quadratic forms of multivariate normal random vectors, we have
    \begin{align*}
        \P_{\bSigma^\star}(W_i^2)
        =2\tr\del{(\bSigma^{-1}\bSigma^\star-\I)^2}
        =2\norm[0]{\bSigma^{-1}\bSigma^\star-\I}_{\textup{F}}^2.
    \end{align*}
Since $\fnorm{\bSigma-\bSigma^\star}^2=\fnorm{\bSigma(\bSigma^{-1}\bSigma^\star-\I)}^2\ge\lambda_p^2(\bSigma^\star)\fnorm{\bSigma^{-1}\bSigma^\star-\I}^2$, by Markov's inequality, we have  
    \begin{equation*}
        \P^{(n)}_{\bSigma^\star}\del{|\bar{W}|>\frac{\eta \sqrt{\log n}}{\sqrt{n}\lambda_p(\bSigma^\star)}} \le \frac{1}{\log n}
    \end{equation*}
for all $\bSigma\in \cE,$
which completes the proof.
\end{proof}

\begin{lemma}
\label{lem:lap_concent}
Assume that $\bbeta:=(\beta_1, \dots, \beta_s)^\top$ is distributed as $\beta_j\iidsim\Lap(1)$ for $j\in[s]$. Then for any $\bbeta_0\in \R^s$ and any $\epsilon>0$, 
	\begin{eqnarray*}
	\P_{\Lap(1)}^{(s)}(\norm{\bbeta-\bbeta_{0}}_2\le \epsilon) 
	\ge \e^{-\norm{\bbeta_0}_1-\epsilon-s\log(s/\epsilon)}.
	\end{eqnarray*}
\end{lemma}

\begin{proof}
By a change of variables $\bbeta-\bbeta_0\to \bbeta^*$, we have
	\begin{align*}
	\P_{\Lap(1)}^{(s)}\del{\norm{\bbeta-\bbeta_0}_1\le\epsilon} 
	&=\int_{\norm{\bbeta-\bbeta_0}_1\le\epsilon} \frac{1}{2}\e^{-\|\bbeta\|_1} \d\bbeta\\
	&\ge \e^{-\norm{\bbeta_0}_1}\int_{\norm{\bbeta^*}_1\le\epsilon} \frac{1}{2}\e^{-\|\bbeta^*\|_1} \d\bbeta^* \\
	&=\e^{-\norm{\bbeta_0}_1}\P_{\EXP(1)}^{(s)}\del{\sum_{i=1}^s E_i\le\epsilon},
	\end{align*}
where $E_1,\dots, E_s$ are iid random variables following exponential distribution with scale $1$. Since the random variable $\sum_{i=1}^s E_i$ follows the gamma distribution with shape $n$ and scale $1$, we have that
	\begin{align*}
	\P_{\EXP(1)}^{(s)}\del{\sum_{i=1}^s E_i\le\epsilon}
	&=\frac{1}{\Gamma(s)}\int_{0}^{\epsilon}x^{s-1}\e^{-x}\d x \\
	&\ge \frac{1}{(s-1)!} \e^{-\epsilon} \int_{0}^{\epsilon}x^{s-1}\d x
	= \frac{\epsilon^s}{s!} \e^{-\epsilon}.
	\end{align*}
The inequality $s!\le s^s$ completes the proof.
\end{proof}

Now we are ready to prove \cref{lem:d_n}. 

\begin{proof}[Proof of \cref{lem:d_n}]
Since $\lambda(\bSigma^\star)\ge\psi_0$ and $\rho(\bSigma^\star)\lesssim \lambda\lesssim n,$  \cref{lem:d_n0} with $\eta:=\sqrt{sr/n}$ implies that
    \begin{equation*}
        D_n\ge \exp\del{-C_1sr\log n}\Pi\del{\fnorm{\bSigma-\bSigma^\star}\le \eta}
    \end{equation*}
for some $C_1>0$ with $\P_{\bSigma^\star}^{(n)}$-probability at least $1-\log^{-1} n$. Since $\B$ and $\psi$ are independent under the prior distribution, 
	\begin{align*}
	&\Pi\del{\fnorm{\bSigma-\bSigma^\star}\le \eta} \ge
	\Pi\left(\fnorm{\B\B^\top-\B^\star(\B^\star)^\top}\le\frac{\eta}{2}\right)
	\Pi\left(\sqrt{p}|\psi-\psi^\star| \le\frac{\eta}{2}\right). 
	\end{align*}
Since $\lambda\ge\psi^\star\ge \psi_0>0$, it follows that
	\begin{align*}
	\Pi\del{|\psi-\psi^\star|\le\frac{\eta}{2\sqrt{p}}}
	&\ge\frac{\eta}{2\sqrt{p}}
	\min_{\psi\in[\psi^\star/2,3\psi^\star/2]}\frac{a_2^{a_1}}{\Gamma(a_1)}\psi^{-a_1-1}\e^{-a_2/\psi} \\
	&\gtrsim  (np)^{-1/2} \lambda^{-a_1-1}\e^{-2a_2/\psi_0} \\
	&\gtrsim \e^{-C_2\log (p\vee n)}
	\end{align*}
for some $C_2>0$. 
Let $\cS^\star:=\supp(\B^\star)$ and let $\cK(\v):=\cbr{k\in[q]:v_k=1}$ for $\v\in\Delta_{q}$.
If $\supp(\B)=\cS^\star$ and $\|\v\|_0=r$, we have
	\begin{align*}
	\fnorm{\B\B^\top-\B^\star(\B^\star)^\top} 
	&=	\fnorm{\B_{[:,\cK(\v)}\B_{[:,\cK(\v)]}^\top-\B^\star(\B^\star)^\top} \\
	&=\fnorm{(\B_{[:,\cK(\v)]}-\B^\star)(\B_{[:,\cK(\v)]}-\B^\star)^\top+2\B^\star(\B_{[:,\cK(\v)]}-\B^\star)^\top}\\
	&\le \fnorm{\B_{[:,\cK(\v)]}-\B^\star}^2 + 2\fnorm{\B^\star(\B_{[:,\cK(\v)]}-\B^\star)^\top}\\
	&\le \fnorm{\B_{[\cS^\star,\cK(\v)]}-\B^\star_{[\cS^\star,:]}}^2 + 2\norm{\B^\star}\fnorm{\B_{[\cS^\star,\cK(\v)]}-\B^\star_{[\cS^\star,:]}}.
	\end{align*}
Therefore, since $\norm[0]{\B^\star}=\lambda_1^{1/2}\del[1]{\B^\star(\B^\star)^\top}\le \lambda^{1/2}\lesssim \sqrt{s}$ by assumption, we have
    \begin{align*}
     \Pi&\del{\fnorm{\B\B^\top-\B^\star(\B^\star)^\top} \ge \eta/2}\\
     &\ge \sum_{\v:\|\v\|_0=r}\Pi(\u^\star, \v)
        \Pi\del{\fnorm{\B_{[\cS^\star,\cK(\v)]}-\B^\star_{[\cS^\star,:]}}^2\ge C_3\sqrt{\frac{r}{n}} \big|\u=\u^\star,\v}
    \end{align*}
for some constant $C_3>0$, where $\u^\star:=(u_j^\star)_{j\in[p]}$ is the binary vector such that $u_j^\star=1$ if $j\in\cS^\star$ and $u_j^\star=0$ otherwise. By \cref{lem:lap_concent}, for $\|\v\|_0=r$, we have
    \begin{align*}
       \Pi&\del{\fnorm{\B_{[\cS^\star,\cK(\v)]}-\B^\star_{[\cS^\star,:]}}^2\ge C_3\sqrt{\frac{r}{n}} \big|\u=\u^\star,\v}\\
        &\ge \exp\del{-C_4sr\log n-\norm[1]{\B^\star_{[\cS^\star,:]}}_1}
        \ge\e^{-C_5sr\log n}.
    \end{align*}
for some constants $C_4>0$ and $C_5>0$, where the second inequality follows from that
        \begin{align*}
        \norm[1]{\B^\star_{[\cS^\star, :]}}_1= &\norm[1]{\B^\star}_1
        \le\sqrt{sr}\norm[1]{\B^\star}_{\textup{F}}\\
        &\le \sqrt{s}r\norm[1]{\B^\star}
        =\sqrt{s}r\lambda_1^{1/2}\del{\B^\star(\B^\star)^\top}
        \le \sqrt{s\lambda}r\lesssim sr.
    \end{align*}
Lastly, we note that
    \begin{align*}
         \sum_{\v:\|\v\|_0=r}\Pi(\u^\star, \v)
        \gtrsim \e^{-Asr\log (p\vee n)} \frac{1}{\binom{p}{s}}
        \ge \e^{-Asr\log  (p\vee n) -s\log p},
    \end{align*}
which completes the proof.
\end{proof}

\subsection{Proof of \cref{lem:test}}
\label{subsec:proof_test}

\begin{proof}[Proof of \cref{lem:test}]
We decompose $H_1$ as
    \begin{equation*}
        H_1\subset \bigcup_{\cS:|\cS|\le  t}\cbr{\bSigma\in\bar\cC(p,q):\norm{\bSigma-\bSigma^\star}\ge M\epsilon, \supp(\B)=\cS}.
    \end{equation*}
For a given loading matrix $\B$ with the support $\cS,$ let $\bar{\cS}:=\cS\cup \cS^\star$, where $\cS^\star:=\supp(\B^\star)$.
Define 
    \begin{align*}
        \bSigma_{\bar{\cS}}&:=\B_{[\bar{\cS},:]}\B_{[\bar{\cS},:]}^\top +\psi\I,\\
        \bSigma^\star_{\bar{\cS}}&:=\B^\star_{[\bar{\cS},:]}\del{\B^\star_{[\bar{\cS},:]}}^\top +\psi^\star\I.
    \end{align*}
Note that $\norm[0]{\bSigma_{\bar{\cS}}-\bSigma^\star_{\bar{\cS}}}=\norm{\bSigma-\bSigma^\star}$ and so  $H_1\subset\cup_{\cS:|\cS|\le  t}H_{1, \cS},$ where we define
    \begin{equation*}
        H_{1, \cS}:=\cbr{\bSigma_{\bar{\cS}}\in\bS_+^{|\bar\cS|}:\norm[0]{\bSigma_{\bar{\cS}}-\bSigma^\star_{\bar{\cS}}}\ge M\tepsilon}.
    \end{equation*}
We  construct a test for $H_0:=\cbr[1]{\bSigma_{\bar{\cS}}^\star}$ versus $H_{1,\cS}$. To do this, we use the following lemma from \cite{gao2015rate}.

\begin{lemma}[Lemma 5.7 of \cite{gao2015rate}]
\label{lem:test_general}
Let $\bSigma_0\in\bS_+^d$. Then for any $M>0$ and $\tepsilon>0$, there is a test function $\tilde{\phi}:\R^{d\times n}\mapsto[0,1]$ such that
    \begin{align*}
    \P_{\bSigma_0}\tilde{\phi}(\Y_{1:n})
    \le \exp\del{C_1d-\frac{C_1M^2}{4\norm{\bSigma_0}^2}n\tepsilon^2}
    + 2\exp(C_1d-C_1M^{1/2}n\tepsilon^2)
    \end{align*}
and
    \begin{align*}
    \sup_{\bSigma_1:\norm{\bSigma_1-\bSigma_0}>M\tepsilon}
    \P_{\bSigma_1}^{(n)}(1-\tilde{\phi}(\Y_{1:n}))
    \le \exp\del{C_1d-\frac{C_1M}{4}\del{1\vee \frac{M}{(M^{1/2}+2)^2\norm{\bSigma_0}^2}}n\tepsilon^2}
    \end{align*}
for some universal constant $C_1>0$. 
\end{lemma}

By the above lemma together with the fact that $\norm[0]{\bSigma^\star_{\bar{\cS}}}=\norm{\bSigma^\star}\le \lambda$, there exists a test function $\phi^{\cS}$ which is a function of $\Y_{1[\bar{\cS}]}, \dots, \Y_{n[\bar{\cS}]}$, such that for any $M>4^{2/3}$
    \begin{align*}
    \P^{(n)}_{\bSigma^\star_{\bar{\cS}}}\phi^{\cS}
    &\le \exp\del{C_1|\bar{S}|-\frac{C_1M^2}{4\lambda^2}n\tepsilon^2}
    + 2\exp(C_1|\bar{\cS}|-C_1M^{1/2}n)\\
    &\le 3\exp\del{C_1|\bar{\cS}|-\frac{C_1M^{1/2}}{\lambda^2\vee1}n\epsilon^2}
    \end{align*}
and
    \begin{align*}
    \sup_{\bSigma_{\bar{\cS}}:\bSigma_{\bar{\cS}}\in H_{1, \cS}}
    \P^{(n)}_{\bSigma^{\bar{\cS}}}(1-\phi^{\cS}) &\le \exp\del{C_1|\bar{\cS}|-\frac{C_1M}{4}n\tepsilon^2}.
    \end{align*}
We combine the tests by $\phi:=\max_{\cS:|\cS|\le  t}\phi^\cS$. Since the test function $\phi^\cS$ depends on the data $\Y_1, \dots, \Y_n$ only through ($j: j\in \bar\cS$)-th coordinates, we have $\P^{(n)}_{\bSigma^\star}\phi^{\cS}=\P^{(n)}_{\bSigma^\star_{\bar{\cS}}}\phi^{\cS}$. Therefore, since $|\bar{\cS}|\le  t+s$ for any $\cS$ with $|\cS|\le  t$, and $\lambda^2\vee1\le (1+1/\psi_0^2)\lambda^2,$ we have
    \begin{align*}
    \P^{(n)}_{\bSigma^\star}\phi&\le \sum_{t=1}^{\floor{ t}}\binom{p}{t} 3\exp\del{C_1(  t+s)-\frac{C_1M^{1/2}}{\lambda^2\vee1}n\tepsilon^2}\\
    &\le 3(  t+1)\e^{(  t+1)\log p} \exp\del{C_1( t+s)-\frac{C_1M^{1/2}}{ (1+\psi_0^{-2})\lambda^2}n\tepsilon^2} \\
      &\le  3\exp\del{2  t\log p+(C_1+1)( t+s)-\frac{C_1M^{1/2}}{(1+\psi_0^{-2})\lambda^2}n\tepsilon^2} 
    \end{align*}
and
    \begin{align*}
    \sup_{\bSigma:\bSigma\in H_1}\P^{(n)}_{\bSigma}(1-\phi) 
    &\le \sup_{\cS:|\cS|\le  t}\sup_{\bSigma_{\bar{\cS}}:\bSigma_{\bar{\cS}}\in H_{1, \cS}}
    \P^{(n)}_{\bSigma_{\bar{\cS}}}(1-\phi^{\cS})\\
    &\le\exp\del{C_1( t+s)-\frac{C_1M}{4}n\tepsilon^2},
    \end{align*}
which completes the proof.
\end{proof}

\section{Additional numerical results}

\subsection{Convergence diagnostics}
\label{sec:conv}

We assess the convergence of our MCMC samples via trace, autocorrelation and partial autocorrelation plots for some randomly selected nonzero loadings. \cref{fig:mcmc_conv1} presents such plots from the Bayesian factor model under the AdaSS prior for the AGEMAP data considered in \cref{subsec:real_data}.   The results show that the MCMC chain converges well. 

\begin{figure}
    \centering
    \makebox[\textwidth][c]{
        \includegraphics[scale=0.24]{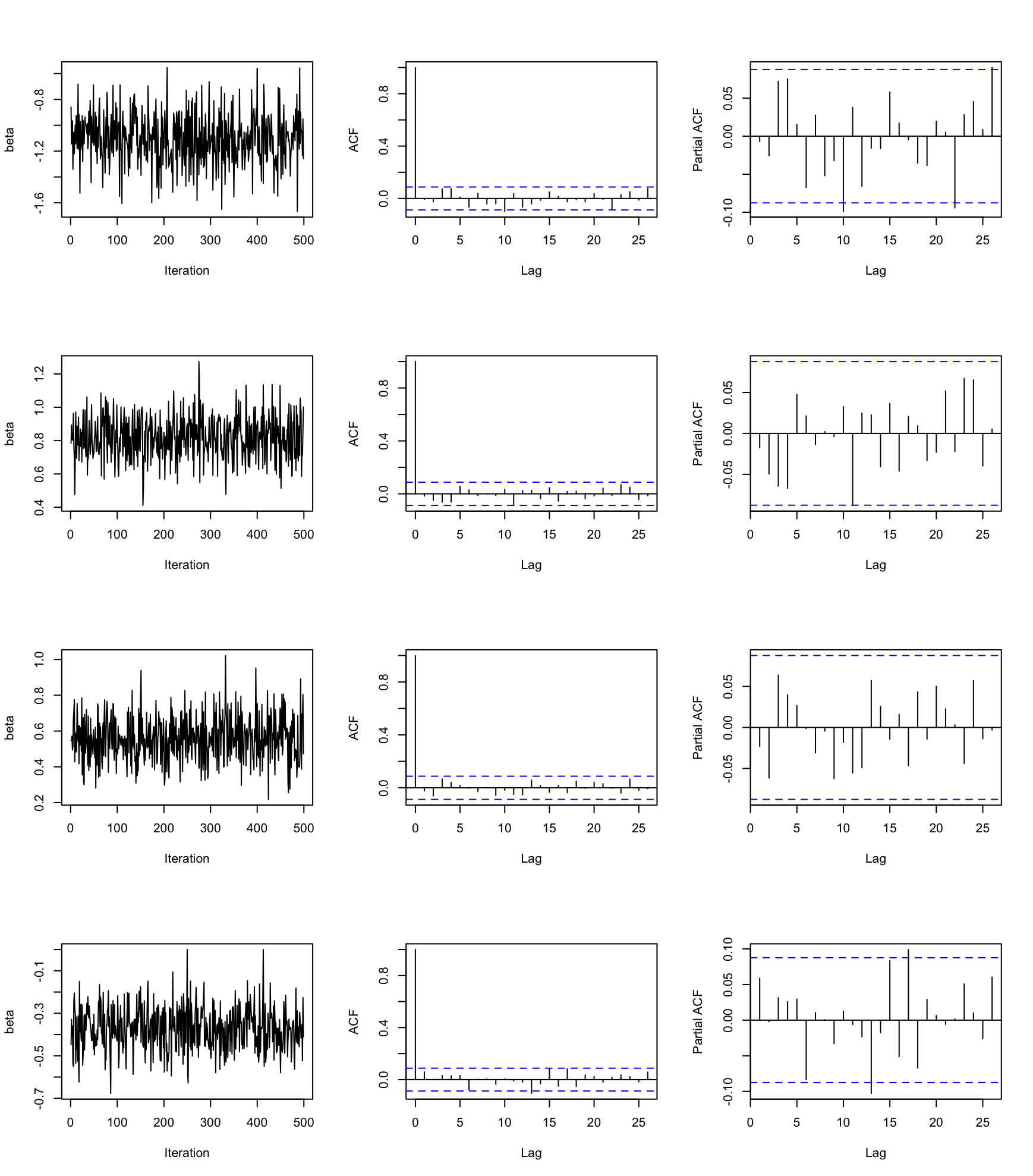}
    }
    \caption{The trace, autocorrelation and partial autocorrelation plots of the posterior samples of some randomly selected nonzero loading for the AGEMAP data.}
    \label{fig:mcmc_conv1}
\end{figure}

\subsection{Application of the post-processing algorithm of \citet{papastamoulis2022identifiability}}
\label{sec:pp}

In this section, we consider the estimation of the loading matrix from the MCMC samples of the proposed Bayesian model using the post-processing scheme of  \citet{papastamoulis2022identifiability}. To illustrate this approach, we consider the following toy example, where we generate a dataset of size $n=100$ from the factor model with the sparse loading matrix $\B^\star=(\beta_{jk}^\star)_{j\in[50],k\in[5]}\in\R^{50\times 5}$ having a ``diagonal pattern'' for non-sparse rows such that $\beta_{jk}^\star=1$ if $5(k-1)< j\le 5k$ for $k\in[5]$ and   $\beta_{jk}^\star=0$ otherwise (see \cref{fig:loading_true}). \cref{fig:loading_raw,fig:loading_pp} present the posterior means of the loading matrix computed before and after post-processing the MCMC samples, respectively. We can see that, without post-processing, the true pattern of the loading matrix is poorly recovered. For instance, the first and 6th columns are overlapped. But after post-processing, we are able to capture the true pattern up to a permutation of columns and signs of loadings.

\begin{figure}[t]
    \centering   
    \begin{subfigure}[c]{0.22\textwidth}
        \includegraphics[width=\textwidth,height=2.9\textwidth]{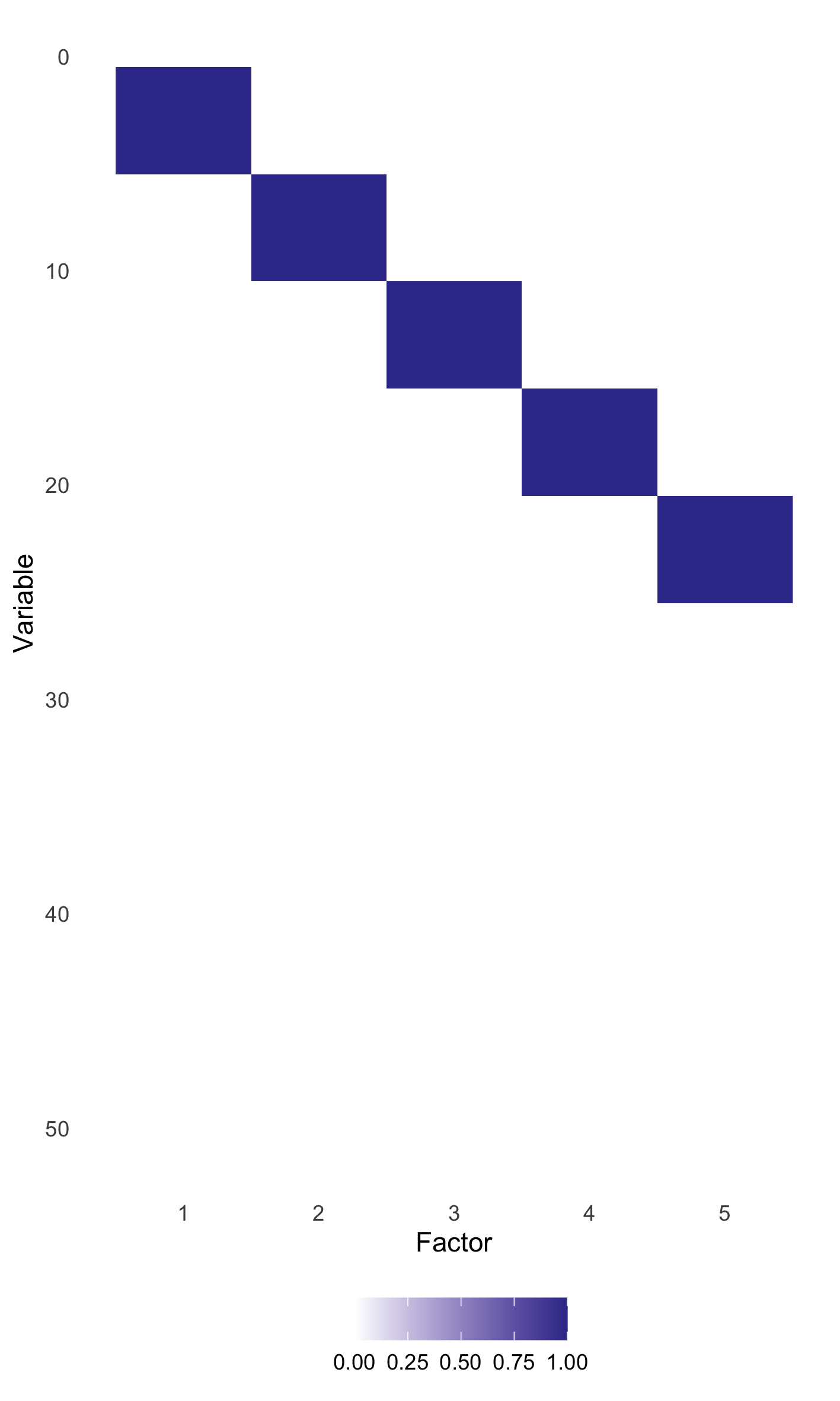}
        \subcaption{True}
        \label{fig:loading_true}
    \end{subfigure}
    \begin{subfigure}[c]{0.38\textwidth}
        \includegraphics[width=\textwidth]{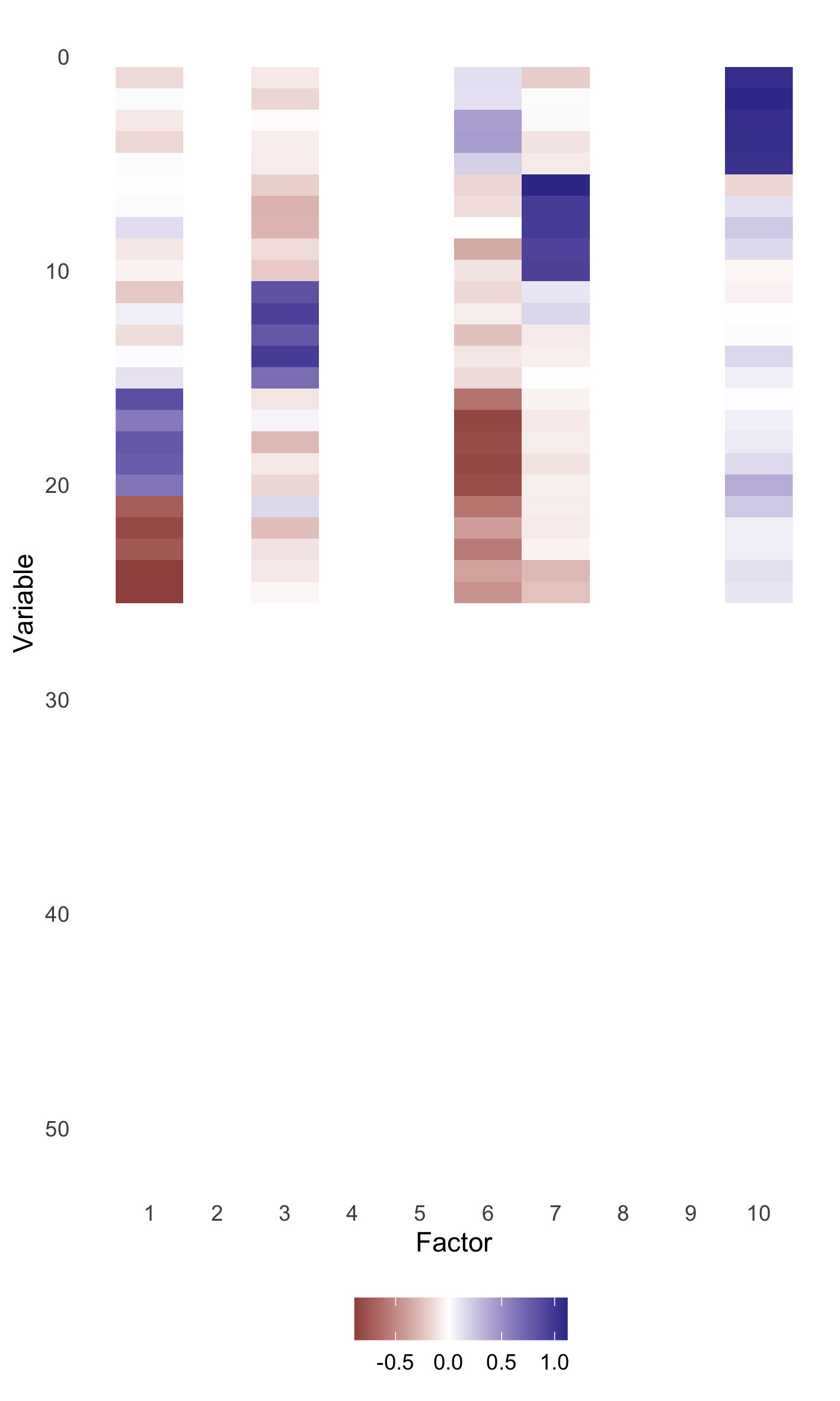}
        \subcaption{Raw}
         \label{fig:loading_raw}
    \end{subfigure}   
    \begin{subfigure}[c]{0.38\textwidth}
        \includegraphics[width=\textwidth]{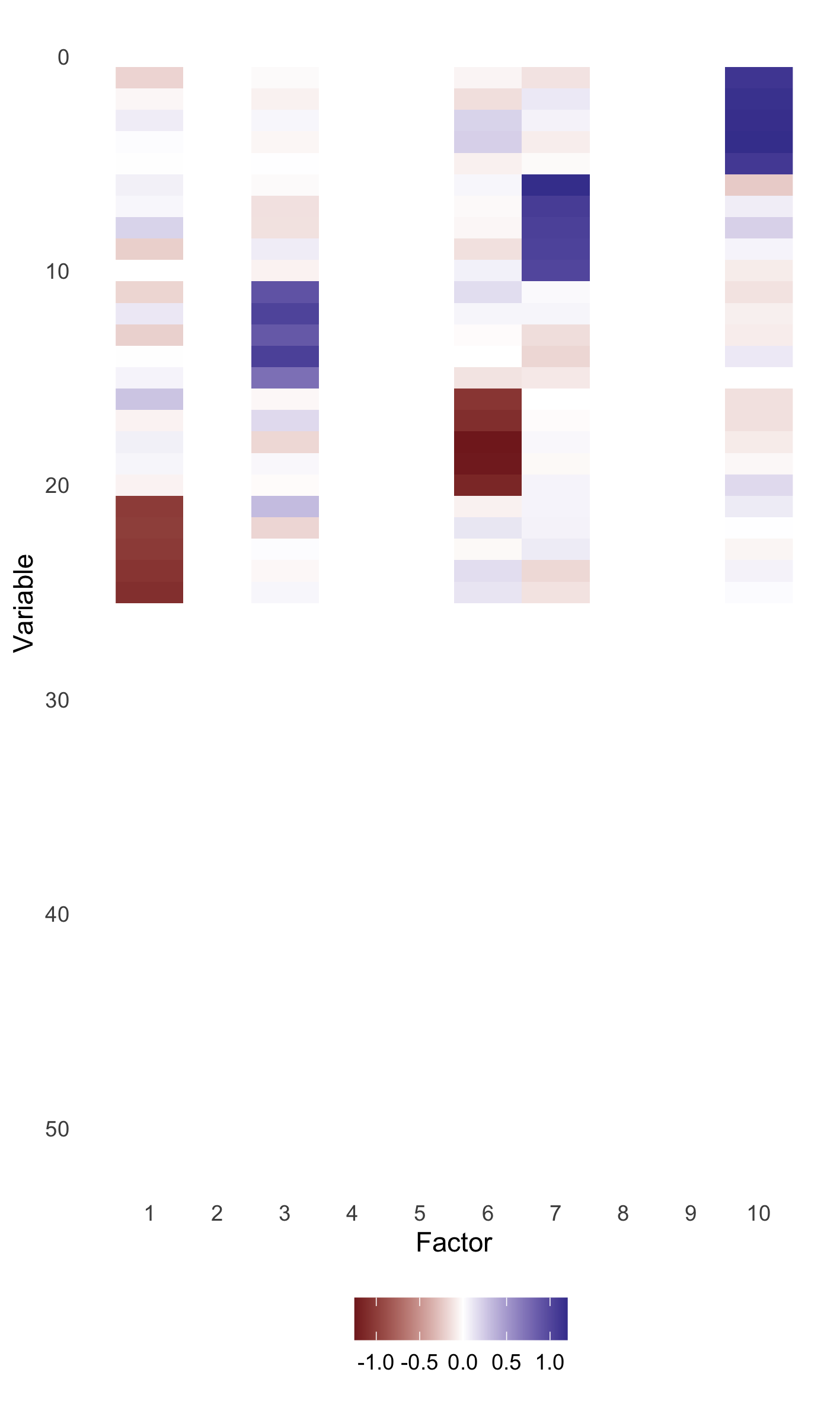}
        \subcaption{Post-processed}
        \label{fig:loading_pp}
    \end{subfigure}
    \caption{The true pattern of nonzero values in the loading matrix (a), the posterior means of the loading matrix without  post-processing (b) and with post-processing (c).}
    \label{fig:loading}
\end{figure}

\subsection{Sensitivity analysis of the choice of $q$ }
\label{sec:sens}

In this section, we conduct a sensitivity analysis regarding to the choice $q$. The setup for the synthetic dataset is the same as that in the \cref{subsec:simulation}. We only consider the case where the sparsity $s=30$ for simplicity. We compare the posterior distributions of the factor dimensionality $\xi$ for the three values of $q\in\{10, 20, 50\}$. \cref{fig:sensitivity} presents the trace plots of  the factor dimensionality. In every case, the trace of the factor dimensionality converges to the true value regardless of the choice of $q$, which suggests that the posterior distribution is insensitive to the choice of $q$.

\begin{figure}
    \centering
    \makebox[\textwidth][c]{
        \includegraphics[scale=0.14]{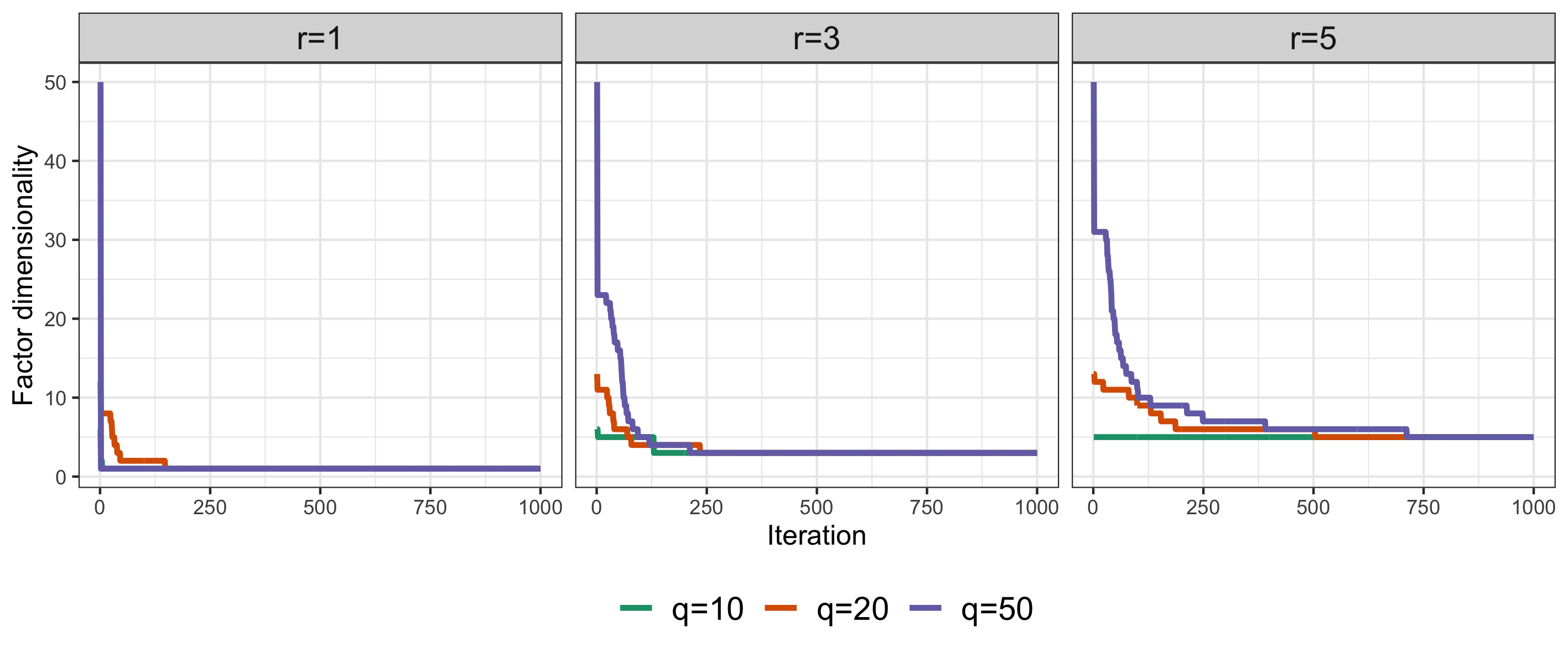}
    }
    \caption{The trace plots of the factor dimensionality for the three values of $q\in\{10, 20, 50\}$.}
    \label{fig:sensitivity}
\end{figure}

\subsection{Simulation for large samples}
\label{sec:large}

In this section, we conduct a simulation study with synthetic data with $n=500>p=300$. The setup for generating the $s$-sparse true loading matrix $\B^\star\in \R^{p\times r}$ is the same as that in the \cref{subsec:simulation}. We let the sparsity $s$ and factor dimensionality $r$ vary among  $s\in\{10, 30, 50\}$ and $r\in\{1,3,5\}$, respectively. \cref{tab:nfac_large,tab:cov_large} present the simulation results for the  factor dimensionality and covariance matrix estimation, respectively, which show that the AdaSS  performs well.

\begin{table}[ht]
\centering
\caption{Proportions of correct estimation (``True''), overestimation (``Over'') and underestimation (``Under'') of the estimated factor dimensionalities for various sparsity $s$ and true factor dimensionality $r$ obtained on 100 synthetic data sets with $n=500$ and $p=300$. ``Ave'' is the average of the estimated factor dimensionalities.}
\label{tab:nfac_large}
\begin{tabular}{ccccccccc} \hline
 $s$ & $r$ &   & ET & ER & GR & ACT & DT & AdaSS \\ \hline
 \multirow{12}{*}{10} & \multirow{4}{*}{1} & True & 83 & 100 & 100 & 94 & 100 & 100 \\ 
   &  & Over & 17 & 0 & 0 & 6 & 0 & 0 \\ 
   &  & Under & 0 & 0 & 0 & 0 & 0 & 0 \\ 
   &  & Ave & 1.19 & 1 & 1 & 1.06 & 1 & 1 \\ 
  \cline{2-9} & \multirow{4}{*}{3} & True & 96 & 73 & 74 & 25 & 81 & 97 \\ 
   &  & Over & 3 & 0 & 0 & 1 & 0 & 2 \\ 
   &  & Under & 1 & 27 & 26 & 74 & 19 & 1 \\ 
   &  & Ave & 3.02 & 2.69 & 2.71 & 2.25 & 2.81 & 3.01 \\ 
  \cline{2-9} & \multirow{4}{*}{5} & True & 91 & 28 & 29 & 0 & 22 & 78 \\ 
   &  & Over & 0 & 0 & 0 & 0 & 0 & 15 \\ 
   &  & Under & 9 & 72 & 71 & 100 & 78 & 7 \\ 
   &  & Ave & 4.9 & 3.76 & 3.86 & 2.05 & 4.09 & 5.1 \\ 
   \hline
\multirow{12}{*}{30} & \multirow{4}{*}{1} & True & 80 & 100 & 100 & 97 & 100 & 100 \\ 
   &  & Over & 20 & 0 & 0 & 3 & 0 & 0 \\ 
   &  & Under & 0 & 0 & 0 & 0 & 0 & 0 \\ 
   &  & Ave & 1.23 & 1 & 1 & 1.03 & 1 & 1 \\ 
  \cline{2-9} & \multirow{4}{*}{3} & True & 98 & 99 & 99 & 94 & 80 & 100 \\ 
   &  & Over & 2 & 0 & 0 & 6 & 0 & 0 \\ 
   &  & Under & 0 & 1 & 1 & 0 & 20 & 0 \\ 
   &  & Ave & 3.02 & 2.99 & 2.99 & 3.06 & 2.8 & 3 \\ 
  \cline{2-9} & \multirow{4}{*}{5} & True & 100 & 97 & 97 & 86 & 21 & 96 \\ 
   &  & Over & 0 & 0 & 0 & 0 & 0 & 2 \\ 
   &  & Under & 0 & 3 & 3 & 14 & 79 & 2 \\ 
   &  & Ave & 5 & 4.97 & 4.97 & 4.86 & 4.12 & 5 \\ 
   \hline
\multirow{12}{*}{50} & \multirow{4}{*}{1} & True & 78 & 100 & 100 & 98 & 100 & 100 \\ 
   &  & Over & 22 & 0 & 0 & 2 & 0 & 0 \\ 
   &  & Under & 0 & 0 & 0 & 0 & 0 & 0 \\ 
   &  & Ave & 1.26 & 1 & 1 & 1.02 & 1 & 1 \\ 
  \cline{2-9} & \multirow{4}{*}{3} & True & 99 & 100 & 100 & 91 & 0 & 98 \\ 
   &  & Over & 1 & 0 & 0 & 9 & 0 & 0 \\ 
   &  & Under & 0 & 0 & 0 & 0 & 100 & 2 \\ 
   &  & Ave & 3.01 & 3 & 3 & 3.09 & 1.12 & 2.98 \\ 
  \cline{2-9} & \multirow{4}{*}{5} & True & 100 & 100 & 100 & 100 & 0 & 97 \\ 
   &  & Over & 0 & 0 & 0 & 0 & 0 & 1 \\ 
   &  & Under & 0 & 0 & 0 & 0 & 100 & 2 \\ 
   &  & Ave & 5 & 5 & 5 & 5 & 3.23 & 4.99 \\ 
   \hline
\end{tabular}
\end{table}

\begin{table}[ht]
\centering
\caption{The averages and standard errors of the scaled spectral norm losses of the estimators of the covariance matrix obtained on 100 synthetic data sets with $n=500$ and $p=300$.} 
\label{tab:cov_large}
\begin{tabular}{cccccccc} \hline
 $s$ & $r$ & POET & SPCA-VI & MGPS & MDP & SSL-IBP & ABayes \\  \hline
 \multirow{3}{*}{10} & 1 & 0.43 (0.046) & \textbf{0.084 (0.026)} & 3.451 (1.311) & 0.343 (0.033) & 0.238 (0.022) & 0.087 (0.029) \\ 
   & 3 & 0.35 (0.049) & 0.114 (0.035) & 1.683 (1.173) & 0.292 (0.039) & 0.218 (0.03) & \textbf{0.1 (0.029)} \\ 
   & 5 & 0.316 (0.053) & 0.132 (0.044) & 0.94 (0.614) & 0.275 (0.045) & 0.227 (0.035) & \textbf{0.125 (0.044)} \\ 
   \hline
\multirow{3}{*}{30} & 1 & 0.432 (0.046) & \textbf{0.119 (0.036)} & 5.511 (2.464) & 0.342 (0.035) & 0.25 (0.032) & 0.131 (0.039) \\ 
   & 3 & 0.41 (0.044) & 0.223 (0.049) & 3.73 (1.341) & 0.33 (0.037) & 0.243 (0.024) & \textbf{0.164 (0.081)} \\ 
   & 5 & 0.379 (0.04) & 0.217 (0.034) & 2.83 (1.931) & 0.314 (0.034) & 0.247 (0.026) & \textbf{0.18 (0.042)} \\ 
   \hline
\multirow{3}{*}{50} & 1 & 0.442 (0.047) & \textbf{0.166 (0.036)} & 5.003 (2.097) & 0.348 (0.036) & 0.264 (0.029) & 0.178 (0.046) \\ 
   & 3 & 0.416 (0.044) & 0.298 (0.054) & 4.346 (1.849) & 0.333 (0.037) & 0.266 (0.029) & \textbf{0.201 (0.042)} \\ 
   & 5 & 0.412 (0.039) & 0.284 (0.036) & 3.799 (2.776) & 0.338 (0.035) & 0.258 (0.024) & \textbf{0.248 (0.087)} \\ 
   \hline
\end{tabular}
\end{table}

\end{document}

%% file: _preamble.tex
\usepackage{amsfonts,amsmath,amssymb,amsthm,mathtools,bbm,bm}
\usepackage{physics,commath}
\usepackage{graphicx,caption,subcaption,hyperref,enumitem,multirow}
\usepackage[dvipsnames]{xcolor}
\usepackage[nameinlink,capitalize,noabbrev]{cleveref}
\usepackage[authoryear]{natbib}
\usepackage[toc,title]{appendix}
\hypersetup{colorlinks=True, citecolor=NavyBlue, linkcolor=MidnightBlue, urlcolor=Blue}
\usepackage{multibib} 
\newcites{S}{References}

\newtheorem{theorem}{Theorem}
\newtheorem{lemma}[theorem]{Lemma}
\newtheorem{proposition}[theorem]{Proposition}

\theoremstyle{definition}

\newtheorem{assumption}{Assumption}
\newenvironment{assumption*}[1]{\renewcommand\theassumption{#1}\assumption}{\endassumption}
\newtheorem{remarkx}{Remark}
\newenvironment{remark}
  {\pushQED{\qed}\remarkx}
  {\popQED\endremarkx}

\numberwithin{equation}{section}
\newenvironment{equations}{\equation\aligned}{\endaligned\endequation}

\makeatletter
\newcommand*{\rom}[1]{\expandafter\@slowromancap\romannumeral #1@}
\makeatother

\newcommand{\floor}[1]{\left\lfloor {#1} \right\rfloor}
\newcommand{\ceil}[1]{\left\lceil {#1} \right\rceil}
\newcommand{\fnorm}[1]{\left\| {#1} \right\|_{\textup{F}}}

\DeclareMathOperator*{\argmax}{\arg\!\max}
\def\iidsim{\stackrel{\textup{iid}}{\sim}}
\def\indsim{\stackrel{\textup{ind}}{\sim}}
\def\zero{\mathbf 0}
\def\one{\mathbf 1}
\def\ind{\mathbbm{1}}

\def\var{\textup{\textsc{Var}}}
\def\kl{\textup{\textsc{KL}}}
\def\supp{\textup{\textsc{supp}}}
\def\diag{\textup{\textsc{diag}}}
\def\tr{\textup{Tr}}
\def\Span{\textup{\textsc{span}}}
\def\rank{\textup{\textsc{rank}}}

\def\Ber{\texttt{\textup{Bernoulli}}}
\def\EXP{\texttt{\textup{Exp}}}

\def\GIG{\texttt{\textup{GIG}}}
\def\IG{\texttt{\textup{IG}}}

\def\Lap{\texttt{\textup{Lap}}}

\def\N{\texttt{\textup{N}}}

\def\d{\textup{d}}
\def\e{\textup{e}}

\def\E{\mathsf{E}}
\def\P{\mathsf{P}}
\def\R{\mathbb{R}}

\makeatletter
\def\greekvectors#1{%
 \@for\next:=#1\do{%
    \def\X##1;{\expandafter\def\csname b##1\endcsname{\bm{\csname##1\endcsname}}}
    \expandafter\X\next;}
 \@for\next:=#1\do{%
    \def\X##1;{\expandafter\def\csname h##1\endcsname{\widehat{\csname##1\endcsname}}}
    \expandafter\X\next;}
 \@for\next:=#1\do{%
    \def\X##1;{\expandafter\def\csname c##1\endcsname{\check{\csname##1\endcsname}}}
    \expandafter\X\next;}
     \@for\next:=#1\do{%
    \def\X##1;{\expandafter\def\csname t##1\endcsname{\tilde{\csname##1\endcsname}}}
    \expandafter\X\next;}
 \@for\next:=#1\do{%
    \def\X##1;{\expandafter\def\csname hb##1\endcsname{\widehat{\bm{\csname##1\endcsname}}}}
    \expandafter\X\next;}
}
\greekvectors{alpha,beta,gamma,delta,epsilon,zeta,theta,kappa,lambda,mu,nu,xi,pi,rho,tau,phi,chi,psi,omega,
    Delta,Gamma,Theta,Lambda,Xi,Pi,Sigma,Phi,Psi,Omega}
    
\@tfor\next:=abcfghijlmnopqrstuvwxyzABCDGHIJLMOQSTUVWXYZ\do{%
    \def\command@factory#1{\expandafter\def\csname #1\endcsname{\mathbf{#1}} }
    \expandafter\command@factory\next
}
\@tfor\next:=abcdeghijklnopqrstuvwxyzABCDEFGHIJKLMNOPQRSTUVWXYZ\do{%
    \def\command@factory#1{\expandafter\def\csname b#1\endcsname{\mathbbm{#1}} }
    \expandafter\command@factory\next
}
\@tfor\next:=ABCDEFGHIJKLMNOPQRSTUVWXYZ\do{%
    \def\command@factory#1{\expandafter\def\csname c#1\endcsname{\mathcal{#1}} }
    \expandafter\command@factory\next
}
\@tfor\next:=abdefghijklmnopqrtuvwxyzABCDEFGHIJKLMNOPQRSTUVWXYZ\do{%
    \def\command@factory#1{\expandafter\def\csname s#1\endcsname{\mathsf{#1}} }
    \expandafter\command@factory\next
}
\@tfor\next:=abcdefghjklmnopqrstuvwxyzABCDEFGHIJKLMNOPQRSTUVWXYZ\do{
    \def\command@factory#1{\expandafter\def\csname f#1\endcsname{\mathfrak{#1}} }
    \expandafter\command@factory\next
}